\documentclass{article}

\usepackage[utf8]{inputenc} 
\usepackage[T1]{fontenc}    
\usepackage{hyperref}       
\usepackage{url}            
\usepackage{booktabs}       
\usepackage{tabularx}       
\usepackage{amsfonts}       
\usepackage{nicefrac}       
\usepackage{microtype}      
\usepackage{xcolor}         
\usepackage{natbib}
\usepackage{graphicx}
\usepackage{subcaption}
\usepackage{hyperref}
\usepackage{amsmath}
\usepackage{amsfonts}
\usepackage{todonotes}
\usepackage{threeparttable}
\usepackage{tablefootnote}
\usepackage{amssymb}
\usepackage{pdflscape}
\usepackage{amsthm}
\usepackage{xspace}
\usepackage{booktabs}
\usepackage{tikz}
\usetikzlibrary{arrows.meta, decorations.pathreplacing, calc, fit, backgrounds, positioning}
\usepackage{adjustbox}
\usepackage{enumitem}
\usepackage{twemojis}
\usepackage{bbm}

 \usepackage[final, main]{neurips_2026}

\usepackage{amsmath}
\usepackage{amssymb}
\usepackage{mathtools}
\usepackage{amsthm}

\usepackage[capitalize,noabbrev]{cleveref}

\theoremstyle{plain}
\newtheorem{theorem}{Theorem}[section]
\newtheorem{proposition}[theorem]{Proposition}

\theoremstyle{definition}
\newtheorem{definition}[theorem]{Definition}

\theoremstyle{remark}

\newcommand{\pvi}{\textsc{pvi}\xspace}
\newcommand{\pmi}{\textsc{pmi}\xspace}

\newcommand{\vinfo}{$I_\mathcal{V}(X \to Y)$\xspace}

\DeclareMathOperator*{\argmax}{arg\,max}

\title{Mecha-nudges for Machines}

\author{%
  Giulio Frey \\
  University of Chicago \\
  \texttt{giulio@uchicago.edu} \\
  \And
  Kawin Ethayarajh \\
  University of Chicago \\
  \texttt{kawin@uchicago.edu} \\
}

\begin{document}

\maketitle

\begin{abstract}
As AI agents make decisions in the same environments as humans, the environments themselves can change to selectively influence them. 
We call this \textit{mecha-nudging}: subtle changes to how choices are presented that systematically influence AI agents without materially degrading the decision environment for humans. It can arise from both intentional and unintentional adaptation to the rise of agents.
To measure this phenomenon, we combine two frameworks---\textit{Bayesian persuasion} from economics and $\mathcal{V}$-\textit{usable information} from computer science---to get a common unit (bits) for quantifying how environments change across a wide range of interventions, contexts, and models.
Applying our framework to over six million product listings on Etsy---a global marketplace for independent sellers---we find that after ChatGPT's release, listings contain significantly more machine-usable information about agentic curation, rising by over $40\%$ of the maximum possible increase. 
This shift is robust across prompts, token choices, and models; absent in a regulated-text placebo; and cannot be explained by the rise of LLM-assisted writing. 
In contrast, a human study finds little to no change in human-usable information. 
Our results provide the first large-scale evidence that mecha-nudging is already happening in the wild, but going unnoticed.
\end{abstract}

\section{Introduction}
\textit{Nudges} are subtle changes to the way choices are presented to human decision-makers, with the goal of shifting their behavior in a particular direction \citep{thalerNudgeImprovingDecisions2008}.
What distinguishes nudges from other interventions is that they must be easy to avoid.
Anything that removes options or noticeably changes economic incentives cannot be a nudge.
For example, placing healthy foods at eye-level to encourage better eating is a nudge; banning or taxing unhealthy food is not.
Nudges work by exploiting limitations of human cognition, such as limited attention and sensitivity to how outcomes are framed.
They have been adopted worldwide by policymakers and businesses alike: for example, in the Dominican Republic, nudges designed to increase tax compliance boosted tax revenue by \$184M (or $\sim$0.22\% of GDP) in 2018 alone \citep{holz2023100}.

For decades, nudges exclusively targeted individuals or groups of humans, as they were the only decision-makers.
This is no longer the case.
Many decisions are now made by AI agents, often in spaces still inhabited by humans. 
With little to no human oversight, they can shortlist job applicants, book travel arrangements, ban content, and more.
As agents become decision-makers in their own right, a natural question follows: can they be nudged too?

We introduce the concept of \textit{mecha-nudges}: changes to how choices are presented that systematically influence the behavior of AI agents without materially degrading the decision environment for humans.
For example, an online seller might add specific product descriptors to a listing---say, \textit{high customer satisfaction}---that may do little to sway a human buyer already looking at the product reviews, but that significantly steer the behavior of a shopping agent.
However, this definition is outcome-based, not intent-based: a mecha-nudge may result from deliberate optimization, imitation of other mecha-nudged content, or even incidental changes.

Mecha-nudges should not be conflated with prompt injection or traditional search engine optimization.
Prompt injection acts directly on the model, making it hard to avoid and often depriving the agent of options it would have ordinarily had \citep{willison2022prompt}; mecha-nudges preserve options and act on the environment instead.
Traditional SEO manipulates machine-readable signals (e.g., backlinks) to influence how information is presented to humans, who remain the final decision-makers \citep{Hagendorff2021}.
Mecha-nudges, by contrast, target AI systems that make decisions themselves.

\begin{figure*}[t]
    \centering
\includegraphics[width=\textwidth]{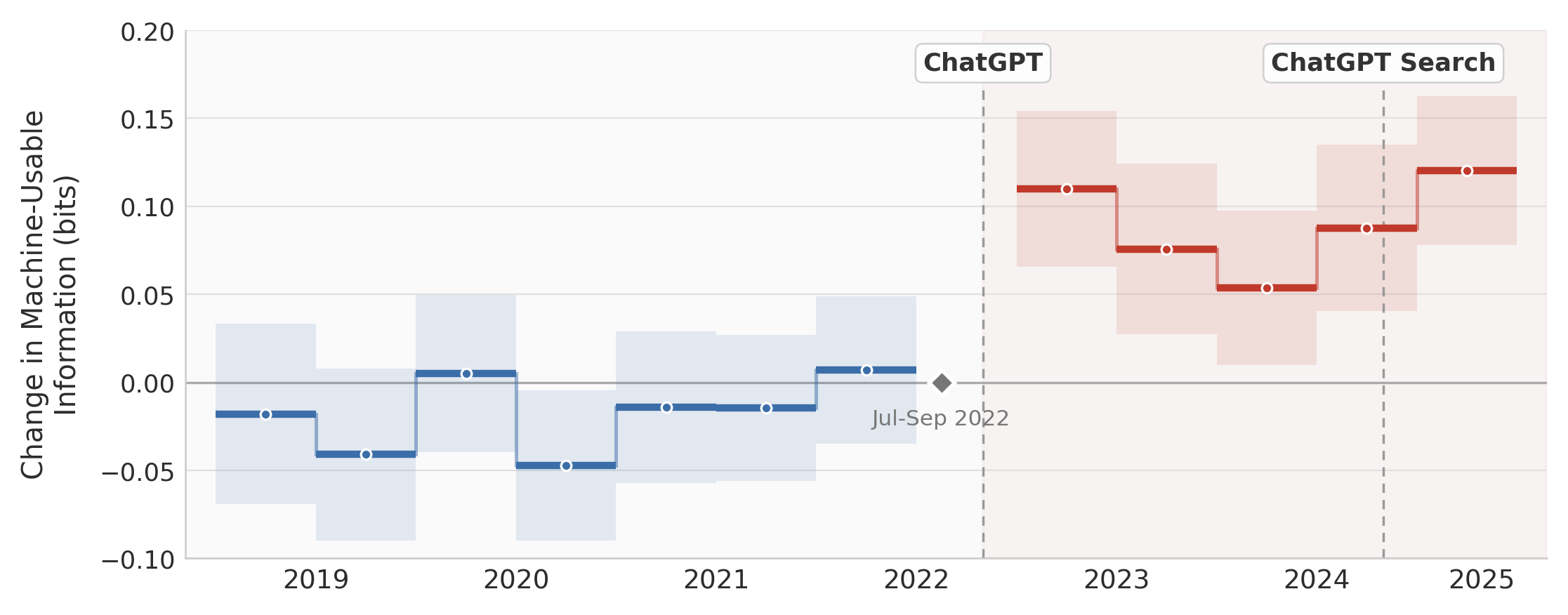}
    \caption{After the release of ChatGPT, the change in machine-usable information in Etsy listings increases significantly.
    The change in each half-year relative to the Jul-Sep 2022 period (grey diamond) is plotted here.
    The effect attenuates over the following year before climbing again in late 2024.
    This coincides with the release of ChatGPT Search, which---unlike its predecessors---could natively browse live listings instead of only making recommendations based on stale training data. }
    \label{fig:time_variation}
\end{figure*}

Although nudges are most often discussed informally in behavioral economics, a common formalization comes from \textit{Bayesian persuasion} \citep{kamenicaBayesianPersuasion2011a}, which explores how a sender might influence the actions of a receiver by controlling the information environment (\S\ref{sec:background}).
We combine this with the notion of \textit{$\mathcal{V}$-usable information}, a generalization of Shannon information that is observer-relative \citep{xutheory,ethayarajh2022understanding}.
Under our formalization, the goal of mecha-nudging is to maximize the amount of machine-usable information that the environment contains about the desired machine behavior, while not decreasing the amount of human-usable information (\S\ref{sec:formalizing}).
This allows all mecha-nudges to be measured on a common scale (bits of usable information), permitting comparisons across different interventions, settings, and models.

There have been anecdotes of mecha-nudging (and evidence that nudges for humans may extend to machines \citep{cherep2025hypersensitive}), but we provide the first large-scale evidence of mecha-nudging by analyzing Etsy, a global marketplace for independent sellers (\S\ref{sec:experiments}).
Etsy is an ideal setting: it has integrated AI-driven features for both buyers and sellers, and a large fraction of referral traffic comes from ChatGPT alone \citep{modernretail2025chatgpt}.
We find that product listings created after the release of ChatGPT have significantly more machine-usable information about agentic curation decisions (Figure \ref{fig:time_variation}), with the machine-usable information jumping $0.143$ bits in the post-period (over $40\%$ of the maximum possible increase).
This shift is robust to prompt formulations, token choices, and model families.
It persists after controlling for product- and seller-level attributes, and is absent in product categories where human buyers are sensitive to AI use (e.g., art and collectibles \citep{millet2023defending}).
Using the number of product reviews as a proxy for cumulative sales, we find that greater machine-usable information is associated with significantly more commercial success in the post-ChatGPT era, but not pre-ChatGPT.
Although our data is observational, these findings are consistent with listings being optimized---whether deliberately or through imitation of successful sellers---to increase machine-usable information.

Mecha-nudges create a deployment feedback loop: once agents make decisions on their own, the environments they inhabit become optimized for them. 
This means model behavior in the wild is determined not only by model developers, but also by the strategic and incidental adaptation of the world in which those models are deployed.
Static evaluations miss this loop.
Our framework makes such adaptation measurable, and our paper provides the first large-scale evidence that it is already happening in the real world.

\section{Background}
\label{sec:background}
We now review nudges, Bayesian persuasion, and $\mathcal{V}$-usable information, the three ingredients of our formalization of mecha-nudges in \S\ref{sec:formalizing}. 

\subsection{Nudges}
\label{ssec:background_nudges}
Early work in behavioral economics found that because cognitive resources are finite, the way in which choices are presented has a significant impact on human decisions \citep{tversky1981framing}.
For example, framing medical treatments in terms of their survival rates instead of their mortality rates makes it far likelier that patients will consent \citep{mcneil1982psychology}.
\citet{thalerNudgeImprovingDecisions2008} introduced the term \textit{choice architecture} to describe the deliberate design of a decision environment, wherein \textit{nudges} denote any part of the architecture that shifts human behavior without removing options or altering incentives.

Nudges have enjoyed broad adoption as a policy tool.
For example, automatic enrollment in U.S. retirement plans increased participation rates for new hires from 49\% to 86\% \citep{madrianPowerSuggestionInertia2001}. In Denmark, automatic contributions to retirement accounts proved far more effective than costly tax subsidies, which generate only one additional cent of saving per dollar of expenditure \citep{chettyActiveVsPassive2014}.
Social comparisons---in the form of letters comparing residents' electricity use to that of their neighbors---decreased energy consumption by $\sim$2\% on average \citep{allcottSocialNormsEnergy2011}.
Recent experimental evidence suggests that LLM agents are more susceptible to choice architecture intended for humans than humans themselves \citep{cherep2025hypersensitive,cherep2026framework}.

Nudges are most often discussed informally, as context-specific interventions.
However, they can be mathematically formalized under some frameworks.
For example, \textit{rational inattention} proposes that hidden signals may not be worth investigating because of information acquisition costs; nudges can be framed as changes in these costs \citep{matejkaRationalInattentionDiscrete2015,simsImplicationsRationalInattention2003}.
Nudges can also be analyzed through the lens of \textit{prospect theory} as changes to the framing of an outcome \citep{tversky1992advances}.
We refer the reader to Appendix \ref{appendix:related} for more details.

\subsection{Bayesian Persuasion}
\label{ssec:background_bayesian}

In information design, the \textit{Bayesian persuasion} framework provides a popular mathematical formalization of nudges \citep{kamenicaBayesianPersuasion2011a,taneva2019information,kamenicaBayesianPersuasionInformation2019a}.
In brief:
\begin{enumerate}[leftmargin=*]
    \item The choice architect (i.e., the sender) and the decision-maker (i.e., the receiver) start with the same prior belief $\mu_0$ about the distribution of some random variable $Z$.
    \item To influence the decision-maker's belief, the choice architect selects a distribution $\pi(\cdot|Z)$ over a possible set of signals $\mathcal{S}$ and commits to it before seeing any instantiation of $Z$.
    \item The decision-maker sees both $\pi$ and an instantiated signal $s \in \mathcal{S}$,  forms a posterior belief $\mu(z|s)$ by applying Bayes' rule, and then takes a utility-maximizing action $a^*$. The choice architect, with utility function $v$, has a final expected utility $\mathbb{E}_{z \sim \mu_0} \mathbb{E}_{s \sim \pi(\cdot|z)} 
    [v(a^*(\mu(\cdot|s)), z)]$.
\end{enumerate}
Under this framework, the nudge is the signal distribution $\pi(\cdot|Z)$.
By choosing how and what information is revealed, the choice architect shapes the decision-maker’s posterior belief---and thus the action they take---without changing the feasible set of actions or payoffs. 
The choice architect's objective is thus to select $\pi \in \Pi$ that maximizes their expected utility.

When the set of possible signals is large and unstructured, as with free-form text, explicitly specifying the signaling scheme and computing exact Bayesian updates becomes intractable.
However, maximizing machine-usable information---which is tractable---can be seen as a bounded-receiver analog of Bayesian Persuasion when the receiver is restricted to a specific model class (Proposition \ref{prop:bounded_bp}).

\subsection{Usable Information}
\label{sec:v-info}

$\mathcal{V}$-usable information is an observer-relative generalization of Shannon information.
To understand why it is useful, consider a model family $\mathcal{V}$ that can learn to map an English sentence $X$ to its French translation $Y$.
If we encrypted $X$, the amount of Shannon information that $X$ contains about $Y$ would be the same, but translating it would be far more difficult because the information it contains is \textit{no longer usable} by $\mathcal{V}$ \citep{xutheory}.
Conversely, if we decrypted the encrypted text, we would increase the amount of usable information.
Although this violates the data processing inequality, it is useful for understanding many real-world phenomena, such as why representation learning is helpful and why some datasets are more difficult to learn from than others \citep{ethayarajh2022understanding}.

\citet{xutheory} propose measuring the amount of $\mathcal{V}$-usable information through a framework called \textit{predictive $\mathcal{V}$-information}. 
We will now re-state the formal definitions in this framework (and its extension to pointwise examples, by \citet{ethayarajh2022understanding}).

\begin{definition}[\textbf{$\mathcal{V}$-Usable Information}]
\label{def:v-info}
Let $X, Y$ denote random variables with sample spaces $\mathcal{X}, \mathcal{Y}$ respectively, and let $\varnothing$ denote a null input that is uninformative about $Y$.
Given predictive family\footnote{A predictive family is a subset of all possible mappings from $\mathcal{X}$ to $P(\mathcal{Y})$ that
satisfies \textit{optional ignorance}: for any $P$ in the range of some $f \in \mathcal{V}$, there exists some $f' \in \mathcal{V}$ s.t. $f[X] = f'[\varnothing] = P$. 
We refer the reader to \citet{xutheory} for a more complete understanding of why optional ignorance is important, as well as for PAC bounds on the estimation error. Neural networks without frozen parameters handily meet this definition, as do human learners.} $\mathcal{V} \subseteq \Omega = \{ f: \mathcal{X} \cup \{\varnothing\} \to P(\mathcal{Y}) \}$, the predictive $\mathcal{V}$-entropy is
\begin{equation}
    H_\mathcal{V}(Y) = \inf_{f \in \mathcal{V}} \mathbb{E} [- \log_2 f[\varnothing](Y) ]
    \label{v-entropy}
\end{equation}
and the conditional $\mathcal{V}$-entropy is
\begin{equation}
    H_\mathcal{V}(Y \mid X) = \inf_{f \in \mathcal{V}} \mathbb{E} [- \log_2 f[X](Y) ]
    \label{cond-v-entropy}
\end{equation}
The $\mathcal{V}$-usable information (or simply $\mathcal{V}$-information) is $I_\mathcal{V}(X \to Y) = H_\mathcal{V}(Y) - H_\mathcal{V}(Y \mid X)$.
\end{definition}

\begin{definition}[\textbf{Pointwise $\mathcal{V}$-Information}]
\label{def:pvi}
Given random variables $X,Y$ and a predictive family $\mathcal{V}$, the pointwise $\mathcal{V}$-information (\pvi) of an instance $(x,y) \sim (X,Y)$ is
\begin{equation}
\label{eq:pvi}
   \text{\pvi}(x \to y) = -\log_2 g[\varnothing](y) + \log_2 g'[x](y)
\end{equation}
where $g = \arg\inf_{f \in \mathcal{V}} \mathbb{E} [- \log_2 f[\varnothing](Y) ]$ and $g' = \arg\inf_{f \in \mathcal{V}} \mathbb{E} [- \log_2 f[X](Y) ]$.
\end{definition}

In brief, $f[X]$ and $f[\varnothing]$ produce a probability distribution over the output space. 
The goal is to find the $f \in \mathcal{V}$ that maximizes the log-likelihood of the output data with the input (\ref{cond-v-entropy}) and without it (\ref{v-entropy}).
For example, say $\mathcal{V}$ is the family induced by a specific LLM such as Llama-3.1-8B-Instruct with no frozen parameters \citep{grattafiori2024llama}.
As in our running example, let $X$ be an English sentence and $Y$ a French sentence.
$g$ would be a model trained to produce French with no context (i.e., a French LLM) and $g'$ would be a model trained to produce the French translation of an English sentence.
The \pvi of an instance $(x,y)$ is then the difference in log-probability these models assign to the true French translation. Analogous to the relationship between \pmi and Shannon information:
\begin{equation}
\label{eq:pmi_pvi}
    \begin{split}
        I(X;Y) &= \mathbb{E}_{x,y \sim P(X,Y)}[\text{\pmi}(x,y)] \\
        I_\mathcal{V}(X \to Y) &= \mathbb{E}_{x,y \sim P(X,Y)}[\text{\pvi}(x \to y)]
    \end{split}
\end{equation}
Unlike $\mathcal{V}$-usable information, \pvi can be negative, indicating that the model performs better by ignoring the input. 
The \pvi of an instance depends on the underlying distribution; the same instance drawn from different distributions will generally yield different \pvi values.
Although one can also take averages of any arbitrary sub-population of the data, it would be imprecise to call their average \pvi the $\mathcal{V}$-usable information, since it is a different distribution than the one used to train the model.

Although \vinfo is asymmetric, when $\mathcal{V}$ is the set of all possible functions, $\mathcal{V}$-usable information reduces to Shannon information.
In practice however, implicit in estimating the $\mathcal{V}$-usable information is the assumption that the data used to find the optimal $f \in \mathcal{V}$ and the held-out data
used to estimate $H_\mathcal{V}(Y \mid X)$ are identically distributed.
Moreover, with large model families such as LLMs, there is no global optimality guarantee; we are assuming that the converged model maximizes the log-likelihood.

The key strength of this framework is that it permits a wide range of comparisons to be done on a common scale (bits of usable information):
\begin{enumerate}[leftmargin=*,label=(\roman*)]
    \item different model families $\mathcal{V}$ by computing \vinfo with the same $X,Y$
    \item different data distributions $(X,Y)$ by computing \vinfo with the same $\mathcal{V}$
    \item different transformations $\tau$ of the input by computing $I_\mathcal{V}(\tau(X) \to Y)$ with the same $\mathcal{V}, X, Y$ 
    \item different instances $(x,y)$ by computing $\textsc{pvi}(x \to y)$ with the same $\mathcal{V}, X, Y$
    \item different slices of data by comparing the mean \pvi within each slice
\end{enumerate}

\section{Formalizing Mecha-nudges}

\label{sec:formalizing}

Consider a shopkeeper who wishes to increase the online sales of a particular product. 
She first redesigns the webpage to better draw attention to its benefits; this nudges humans to buy it.
But she does not want \textit{every} visitor to buy the product either---she wants those with larger budgets to choose higher-margin alternatives, for example.
The `select' vs.\ `pass' decision she wants people to make is a random variable $Y$ conditioned on random variable $X$ (representing both the decision environment $E$ she can control, like the webpage, and the buyer characteristics $U$ she cannot).
Increasingly, she finds that AI agents are either buying the product themselves or recommending it to humans, not merely re-ordering the options for humans to ultimately decide, as in SEO.
She must transform the webpage so that the AI agents also decide in the desired manner, but without putting off humans.

We formalize the shopkeeper's dilemma as the transformation of a decision environment that maximizes machine-usable information while not materially reducing human-usable information.

\begin{definition}[\textbf{Mecha-nudging Design}]
\label{def:mecha-nudging}
Let random variable $X=(E,U) \in \mathcal{X}$, with $E\in\mathcal{E}$ (controllable environment) and $U\in\mathcal{U}$ (uncontrollable exogenous characteristics).
Let random variables $Y_H, Y_M \in \mathcal{Y}$ represent the decision the choice architect wants the human and machine decision-makers to make respectively.
Let $\mathcal{H}, \mathcal{M}$ denote the predictive families for decision-making by humans and machines, and $\epsilon \in \mathbb{R}_{\geq 0}$ the tolerable decrease in human-usable information.
Where $\mathcal{T} \subseteq \{\tau:\mathcal{X}\to\mathcal{X}\}$ is the set of available transformations, all satisfying $\tau(E,U) = (\tau_E(E),U)$ for some map $\tau_E:\mathcal{E}\to\mathcal{E}$, the choice architect's problem is:
\begin{equation}
\begin{gathered}
\argmax_{\tau \in \mathcal{T}} \; I_{\mathcal{M}}(\tau(X) \to Y_M) \\
\text{s.t. } I_{\mathcal{H}}(\tau(X) \to Y_H) \geq I_{\mathcal{H}}(X \to Y_H) - \epsilon
\end{gathered}
\label{eq:design_problem}
\end{equation}
This is \textit{constrained mecha-nudging design}.
When $\epsilon \geq I_{\mathcal{H}}(X \to Y_H)$, the constraint is trivial (as usable information is non-negative) and the problem is one of \textit{unconstrained mecha-nudging design}.
\end{definition}

The constraint in (\ref{eq:design_problem}) asserts that in applying the mecha-nudge, the amount of human-usable information must not decrease by more than $\epsilon$ bits.
For example, a transformation that converts a modern webpage into a text-only listing might make it easier for an AI agent
to parse (higher $I_{\mathcal{M}}(\tau(X) \to Y_M)$) but may create a worse browsing experience for humans (lower $I_{\mathcal{H}}(\tau(X) \to Y_H)$). 
Note that the constraint does not concern what the typical human \textit{would} do, but rather what humans in the aggregate \textit{could} do; this is not a behavioral constraint, but a normative one. 

Implicit in this constraint is the assumption that humans and AI agents are operating in the same decision environment.
Operationalizing (\ref{eq:design_problem}) may require surveying humans or a proxy for  $\mathcal{H}$, such as a learned model of human behavior \citep{santurkar2023whose}. 
In some applications, one could instead justify the constraint institutionally---for example through market or regulatory pressures---or relax it entirely by choosing $\epsilon$ large enough to make it non-binding.

\begin{proposition}[\textbf{Bounded-Receiver Bayesian Persuasion}]
\label{prop:bounded_bp}
Consider a bounded-receiver analog of Bayesian persuasion in which both the choice architect and decision-maker have log-scoring utility $\log_2(\cdot)$, and the decision-maker is restricted to predictive family $\mathcal{M}$.
Then $\argmax_{\tau \in \mathcal{T}} I_{\mathcal{M}}(\tau(X)\to Y_M)$, the solution to unconstrained mecha-nudging, also maximizes the best achievable expected utility for the decision-maker.
\end{proposition}
The proof is given in Appendix \ref{appendix:proofs}.

In observational settings, the choice architect’s latent objective is typically unobserved.
Moreover, a realized mecha-nudge need not be deliberate; it may arise through direct optimization, imitation of successful content, or other endogenous adaptation.
We therefore study realized mecha-nudging with respect to a focal machine action $y^*$ that represents some specific behavior of interest:

\begin{definition}[\textbf{Realized Mecha-nudge}]
\label{def:realized_mecha_nudge}
Let $A_H, A_M \in \mathcal{Y}$ denote the actions taken by the human and machine decision-makers. 
Let $\phi: \mathcal{Y} \to \mathcal{Z}$ be a task-relevant coarsening of the machine action, and define $Z_M = \phi(A_M)$.
A transformation $\tau^* \in \mathcal{T}$ is a realized mecha-nudge with respect to $Z_M$ if
\begin{equation}
I_{\mathcal{M}}(\tau^*(X) \to Z_M)
>
I_{\mathcal{M}}(X \to Z_M)
\label{eq:realized_machine}
\end{equation}
and
\begin{equation}
I_{\mathcal{H}}(\tau^*(X) \to A_H)
\ge
I_{\mathcal{H}}(X \to A_H) - \epsilon.
\label{eq:realized_human}
\end{equation}
The identity map $\phi(A_M)=A_M$ gives a multi-action notion of mecha-nudging; the binary map $\phi(A_M) = 1\{A_M = y^*\}$ recovers the focal-action setting used in our empirical analysis.
\end{definition}

Importantly, a realized mecha-nudge need not increase the marginal probability of any particular action.
It could make selection more likely, rejection more likely, or simply sharpen a distribution with multiple possible actions.
Our goal is not necessarily persuasion in a single direction as much as a systematic influence on machine behavior that is mediated through environmental changes.

\begin{figure*}[tbp]
\centering
\resizebox{\textwidth}{!}{%
\begin{tikzpicture}[
    >={Stealth[length=2mm]},
    arrow/.style={->, thick, color=gray!60},
    stem/.style={thick, color=gray!60},
    databox/.style={
        rectangle, draw=#1!60, rounded corners=4pt,
        minimum width=2.4cm, minimum height=0.75cm,
        align=center, font=\footnotesize\sffamily,
        fill=#1!12, text=black!80
    },
    databox/.default=gray,
    smallbox/.style={
        rectangle, draw=#1!50, rounded corners=3pt,
        minimum width=1.4cm, minimum height=0.55cm,
        align=center, font=\scriptsize\sffamily,
        fill=#1!18, text=black!80
    },
    smallbox/.default=gray,
    tinylabel/.style={font=\tiny\sffamily\itshape, text=black!55},
]

\node[databox=orange, minimum width=2.4cm, minimum height=0.85cm]
    at (0, 0) (data) {Etsy Listing\\[-1pt]$X$};

\node[databox=orange, minimum width=2.8cm, minimum height=0.85cm,
      fill=orange!30, draw=orange!50!orange!40]
    at (5.5, 0) (llm) {Curation Decision\\$Z_M$};

\draw[arrow] (data) -- node[above=2pt, font=\tiny\sffamily,
    fill=white, inner sep=2pt, rounded corners=2pt, draw=gray!75, text=gray!75]{GPT-5-mini}
node[below=5pt, font=\tiny\sffamily,
    fill=white, inner sep=2pt, rounded corners=2pt, draw=gray!75, text=gray!75]{balance classes} (llm);

\coordinate (bl) at ([xshift=-3pt]data.south west);
\coordinate (br) at ([xshift=3pt]llm.south east);
\draw[decorate, decoration={brace, amplitude=6pt, mirror},
      thick, color=gray!45] (bl) -- (br);

\coordinate (bmid) at ($(bl)!0.5!(br) + (0,-6pt)$);
\coordinate (split) at ($(bmid) + (0,-0.3)$);
\draw[stem] (bmid) -- (split);


\node[databox=blue, minimum width=2.2cm, fill=blue!8, draw=blue!35]
    at (-7, -2.2) (pre) {Pre-Period\\[-1pt]$<$ Nov 30, 2022};

\node[smallbox=blue, fill=blue!20, draw=blue!40]
    at (-4.3, -1.65) (pre-content) {Content\\[-1pt]Model};
\node[smallbox=blue, fill=blue!12, draw=blue!30]
    at (-4.3, -2.75) (pre-null) {Null\\[-1pt]Model};

\node[tinylabel, above=0pt of pre-content] {Full Text};
\node[tinylabel, below=0pt of pre-null]    {Empty Text};

\node[smallbox=blue, fill=blue!28, draw=blue!50, minimum width=1.5cm]
    at (-2.2, -2.2) (pvi-pre) {$\text{PVI}_{\text{pre}}$};

\draw[arrow] (pre.east) -- ++(0.35,0) |- (pre-content.west);
\draw[arrow] (pre.east) -- ++(0.35,0) |- (pre-null.west);
\draw[arrow] (pre-content.east) -- ++(0.3,0) |- (pvi-pre.west);
\draw[arrow] (pre-null.east)    -- ++(0.3,0) |- (pvi-pre.west);

\node[databox=red, minimum width=2.2cm, fill=red!8, draw=red!35]
    at (7, -2.2) (post) {Post-Period\\[-1pt]$\geq$ Nov 30, 2022};

\node[smallbox=red, fill=red!20, draw=red!40]
    at (4.3, -1.65) (post-content) {Content\\[-1pt]Model};
\node[smallbox=red, fill=red!12, draw=red!30]
    at (4.3, -2.75) (post-null) {Null\\[-1pt]Model};

\node[tinylabel, above=0pt of post-content] {Full Text};
\node[tinylabel, below=0pt of post-null]    {Empty Text};

\node[smallbox=red, fill=red!28, draw=red!50, minimum width=1.5cm]
    at (2.2, -2.2) (pvi-post) {$\text{PVI}_{\text{post}}$};

\draw[arrow] (post.west) -- ++(-0.35,0) |- (post-content.east);
\draw[arrow] (post.west) -- ++(-0.35,0) |- (post-null.east);
\draw[arrow] (post-content.west) -- ++(-0.3,0) |- (pvi-post.east);
\draw[arrow] (post-null.west)    -- ++(-0.3,0) |- (pvi-post.east);

\draw[arrow] (split) -| (pre.north);
\draw[arrow] (split) -| (post.north);

\node[databox=green, fill=green!10, draw=green!45!black!30,
      minimum width=4.2cm, minimum height=0.85cm]
    at (0, -4) (ols) {OLS Regression\\[-1pt]
    $\text{PVI}_i = \alpha + \beta\;\text{after}_i + \varepsilon_i$};

\draw[arrow] (pvi-pre.south)  -- ++(0,-0.25) -| ([xshift=-6pt]ols.north);
\draw[arrow] (pvi-post.south) -- ++(0,-0.25) -| ([xshift=6pt]ols.north);

\begin{scope}[on background layer]
    \node[draw=blue!35, dashed, rounded corners=6pt,
          fit=(pre)(pre-content)(pre-null)(pvi-pre),
          inner sep=9pt, fill=blue!4] {};
    \node[draw=red!35, dashed, rounded corners=6pt,
          fit=(post)(post-content)(post-null)(pvi-post),
          inner sep=9pt, fill=red!4] {};
\end{scope}

\end{tikzpicture}
}
\caption{Our pipeline for estimating the change in usable information between the pre- and post-ChatGPT periods: construct decision labels $Z_M$, train content and null models for each period, and run an OLS regression of the pointwise $\mathcal{V}$-information (\pvi). This describes our baseline experiment, of which we run many variations (including with controls).}
\label{fig:methodology}
\end{figure*}
\section{Systematic Evidence of Mecha-nudging}
\label{sec:experiments}

We now provide empirical evidence that realized mecha-nudging is already occurring at scale.
The shopkeeper's dilemma that motivated our formalization in \S\ref{sec:formalizing} is not a mere hypothetical.
As such, we study product listings on Etsy, a global marketplace for independent sellers.
Etsy is uniquely exposed to AI: not only does a large fraction of referral traffic come from ChatGPT \citep{modernretail2025chatgpt}, but it was the first platform to enable its products to be purchased directly in ChatGPT \citep{openai2025buyitinchatgpt}.

We employ a three-step pipeline (Figure~\ref{fig:methodology}): (i) construct an agent curation decision $Z_M$ with GPT-5-mini as a proxy for ChatGPT\footnote{We use the Jan 2026 version of the model.}; (ii) finetune an open-weights model (Llama-3.1-8B-Instruct by default) to create the content and null models $g', g$; (iii) regress listing-level \pvi on the period in which it was created, along with other controls.
For each temporal partition used in the analysis (pre/post and, where applicable, half-year bins), we fit separate content and null models on that period’s data and compute \pvi (Definition \ref{def:pvi}) on held-out data to get pointwise estimates of usable information.
We then run a human study on the same listings to test whether the human-side constraint (\ref{eq:realized_human}) is satisfied.

Our empirical estimand is a change in machine-usable information about agentic decisions, not a change in the unconditional rate of a particular action. 
Higher $I_M(X \to Z_M)$ means the listing text makes the agent’s decision more predictable relative to a null input. 
This could be because attractive listings become easier to select, unattractive listings become easier to reject, or both. 
Note that because we use observational data, our empirical analysis does not attempt to establish seller intent, nor is seller intent needed for a transformation to be a mecha-nudge; it could arise, for example, from some sellers simply imitating their more successful counterparts.

\subsection{Data}
\label{sec:data}

The raw dataset (Appendix \ref{appendix:data}) comprises over six million Etsy listings, with 1.06 million uploaded pre-ChatGPT and 5.0 million uploaded post-ChatGPT (November 30, 2022)\footnote{Since the data is a Nov 2025 snapshot, we observe listings at scrape time rather than at initial creation; any post-creation edits to older listings would, if anything, understate the extent of mecha-nudging.}. 
In our baseline specification, $X$ is the controllable textual content.
Other characteristics such as the price, the number of reviews, the average rating, and more are used as controls. 
We operationalize the focal machine action $y^*$ by prompting GPT-5-mini, a proxy for the basic version of ChatGPT used by most consumers, to issue a select/pass decision for each listing\footnote{As robustness checks, we also generate labels with Gemma-3-27B-IT and Qwen3-32B, and try different pairs of words for the select/pass decision; see Appendix~\ref{appendix:label_construction}. As also detailed in Appendix \ref{appendix:label_construction}, we subsample the data to balance the class distribution, as large imbalances lead to noisy estimates of the usable information.}; the resulting indicator serves as $Z_M$.
Note that this output variable is constructed specifically to study agentic curation decisions.
If one wanted to study the impact on some other aspect of agent behavior (e.g., agent refusals due to safety concerns), one would have to construct a different set of labels.

\subsection{Methods}

\paragraph{Human-Usable} We first test the human-side constraint (\ref{eq:realized_human}) indirectly via a contrapositive argument. 
If human-usable information had noticeably declined, then we would expect observable human outcomes to have deteriorated as well. 
Several independent proxies indicate that this has not occurred: (1) marketplace-level spending has remained stable, with Gross Merchandise Sales per active buyer on the Etsy marketplace ranging from roughly \$117 to \$136 on a trailing 12-month basis between 2020 and 2025 \citep{etsy_earnings_2025}; 
(2) buyer engagement remained consistent, with repeat buyers accounting for roughly 47–49\% of active buyers throughout this period \citep{etsy_10k_2025}; (3) buyer surveys found that the importance Etsy shoppers place on product descriptions has seen little change, with upwards of 90\% saying that they were very or somewhat important \citep{erank2023,erank2024}.

We then directly estimate the change in human-usable information with respondents from Prolific, on the same listing sample. 
Respondents evaluated balanced pre/post listing batches and self-elicited the probability that a human predictor with access to the listing could recover their chosen action (Appendix \ref{appendix:human_eval}).
We use self-elicitation because the $\mathcal{V}$-entropy is an infimum over achievable entropies, and any individual will have more context about themselves than a third-party predictor, especially since privacy concerns limit how much context we can disclose to a third-party predictor.
Among valid respondents, post-ChatGPT listings have $0.015$ \textit{fewer} bits of human-usable information than pre-ChatGPT listings, with a 95\% confidence interval of $[-0.044,0.014]$.
Under the stable marginal-action assumption described in Appendix \ref{appendix:human_eval}, this is consistent with negligible human-side degradation, albeit not dispositive.
Thus the post-ChatGPT shift is not just a general improvement in listing quality for both humans and machines; it is machine-skewed, while any human-side degradation appears bounded.

\paragraph{Machine-Usable} To estimate the machine-usable information that $X$ contains about $Z_M$, we need to find the (conditional) $\mathcal{V}$-entropy-minimizing $f \in \mathcal{V}$.
When $\mathcal{V}$ is a neural network, as in our case, this is done in the literature by training a model to predict $Z_M$ with $X$ and without $X$ to get the content model $g'$ and null model $g$ respectively.
However, since we cannot train ChatGPT or even a GPT-5 proxy, we use Llama-3.1-8B-Instruct. 
We verify that the choice of fine-tuning model does not drive our conclusions through robustness checks on models from distinct training lineages (Figure \ref{fig:main_results}). 

We then estimate the \pvi of each example $(x,y)$ in our held-out data using (\ref{eq:pvi}). 
To determine whether the release of ChatGPT was followed by a rise of systematic mecha-nudging, we run a simple OLS regression of \pvi on a binary post-ChatGPT indicator:
\begin{equation}
	\mathrm{\pvi}_i \;=\; \alpha + \beta\,\mathrm{after}_i \;+\; \varepsilon_i ,
	\label{eq:pvi_reg}
\end{equation}
where $\mathrm{after}_i = 1$ for listings created after the release of ChatGPT (November 30, 2022) and $0$ for those created before. The coefficient $\beta$ captures the average difference in \pvi between the two periods; $\beta$ being positive and statistically significant would be evidence of mecha-nudging.

We then run several other regressions to develop a more nuanced picture of any mecha-nudging.
First, we replace the binary indicator with half-year dummies to trace how \pvi evolves over time. 
Second, we control for possible confounders such as price, log number of shop and item reviews, average rating, and a discount indicator.
Third, we model interaction effects between the post-ChatGPT indicator and the product categories ($\mathrm{after}_i \times \mathrm{category}_i$) to assess how mecha-nudging has diffused through the market.
Finally, we include price directly in the labeling prompt, so the model observes both text and price when issuing its decision.
Note that our estimate of $\beta$ should be interpreted as a conditional difference in means, not as a causal effect.

\subsection{Results}
\label{sec:results}

\begin{figure*}[t]
	\centering
	\includegraphics[width=\textwidth]{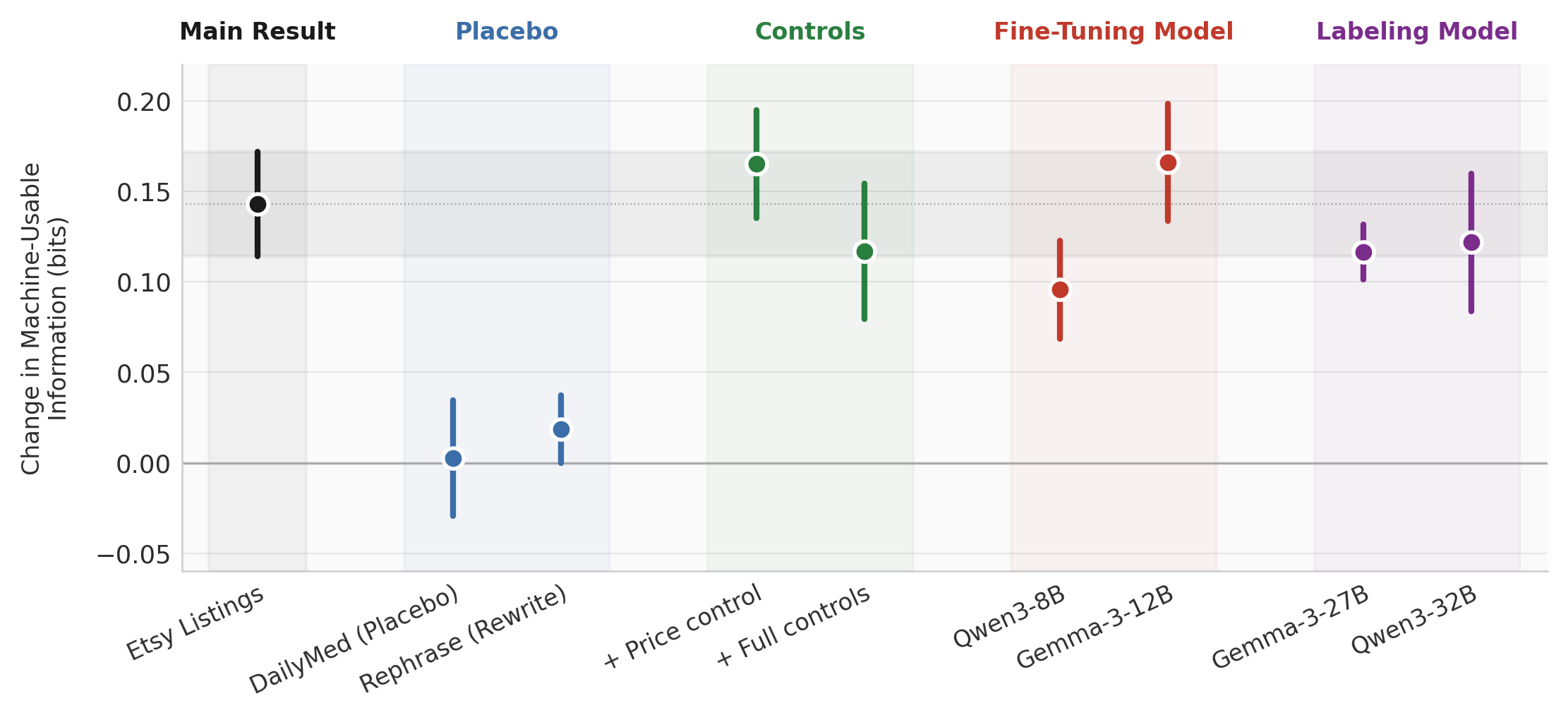}
	\caption{The increase in machine-usable information post-ChatGPT is robust to possible confounders: a generic temporal change (DailyMed), AI-written text (Rephrase), controls for product- and seller-specific attributes (green), the model family that is fine-tuned to estimate \pvi (red), and the LLM used to generate training labels (purple), among others (Appendix \ref{appendix:label_construction}, \ref{appendix:controls}).
    Unless otherwise specified, we use GPT-5-mini as the labeling model and Llama-3.1-8B-Instruct as the fine-tuning model.
    Each point reports the OLS estimate of the post-ChatGPT shift in \pvi, with 95\% CI.}
	\label{fig:main_results}
\end{figure*}

\paragraph{Machine-usable information rises significantly after the release of ChatGPT.} 
Etsy listings created following the release of ChatGPT contain significantly more machine-usable information than those made before, with an estimated increase of $0.143$ bits ($p < 0.01$) out of a maximum possible increase of $0.355$.
This ceiling was calculated by taking 1 bit (the maximum usable information for a class-balanced binary decision, by construction) and subtracting from it the amount of machine-usable information in the pre-ChatGPT period, which was 0.645 bits (Appendix \ref{appendix:data}). 
Note that the latter is non-zero because even non-mecha-nudged text, through disclosing features like product quality, can be predictive of agentic curation.

The temporal dynamics are also informative.
As seen in Figure \ref{fig:time_variation}, compared to the Jul-Sep 2022 reference period, previous half-years did not contain any more machine-usable information, with half-year coefficients fluctuating around zero.
Immediately after the release, machine-usable information sees a sharp and statistically significant jump in the subsequent half-year.
The quickness of this response is not surprising: officials in several Western countries reported a jump in submission volumes for various government services soon after the release of ChatGPT \citep{schmitz2026characterizing}.
The amount of machine-usable information then steadily declines, which---among other possible mechanisms---is consistent with sellers realizing that their efforts have no short-term reward as LLMs are only pulling from historical training data, not live listings.
However, upon the release of ChatGPT Search in Oct 2024, which can browse live listings, \pvi starts steadily climbing again, reaching its peak in the most recent half-year in our data (Jan-Jun 2025).

\paragraph{This result is robust to a wide range of possible confounders.}
As seen in Figure \ref{fig:main_results}, varying token pairs, prompt formulations, labeling models, and fine-tuning models all yield positive and significant estimates of the change in usable information, with most point estimates of $\beta$ falling within the range $[0.09, 0.17]$.
Adding all the listing-level controls depresses the coefficient from $0.143$ to $0.117$, yet it remains significant ($p < 0.01$)---reassuring given that some controls, such as review counts, are themselves equilibrium outcomes (see Appendix \ref{appendix:controls} for details).
        
\paragraph{The effect is not replicated with LLM-rewritten listings.} 
To test whether AI-assisted copywriting might explain these results, we consider two placebos.
First, we rephrase Etsy listings predating the ChatGPT release (i.e., the start of widespread LLM usage by the public) using GPT-5-mini, preserving all factual content while allowing the model to alter wording and style. 
The estimated increase in \pvi is only $0.019$, an order of magnitude below the baseline effect: mechanical rephrasing does not replicate the information gains that mecha-nudging achieved.
We also estimate how much Etsy text is machine-written to begin with. 
A widely-used commercial AI-text detector flags $< 4\%$ of post-ChatGPT listings as majority AI-written, and the rate barely varies with \pvi (Appendix \ref{appendix:ai_detection}).
Second, we replicate our analysis on pharmaceutical drug descriptions from DailyMed, a dataset where text is written according to standardized templates, making it much more difficult for ChatGPT adoption to alter the content (Appendix \ref{appendix:data}).
The estimated shift in \pvi post-ChatGPT is indistinguishable from zero, ruling out a generic temporal trend (Figure \ref{fig:main_results}).
Moreover, the variation in effect size across product categories also suggests selective adaptation: in areas where buyers are sensitive to AI use (e.g., books, music, art, and collectibles) \citep{castelo2019task}, the effect size is indistinguishable from zero; for generic consumer goods like electronics, it is above average (Figure \ref{fig:token_ablations_and_categories}, right).

\begin{figure*}[t]
    \begin{minipage}[t]
    {0.5\textwidth}
        \vspace{0pt}
        \centering
        \includegraphics[width=\linewidth]{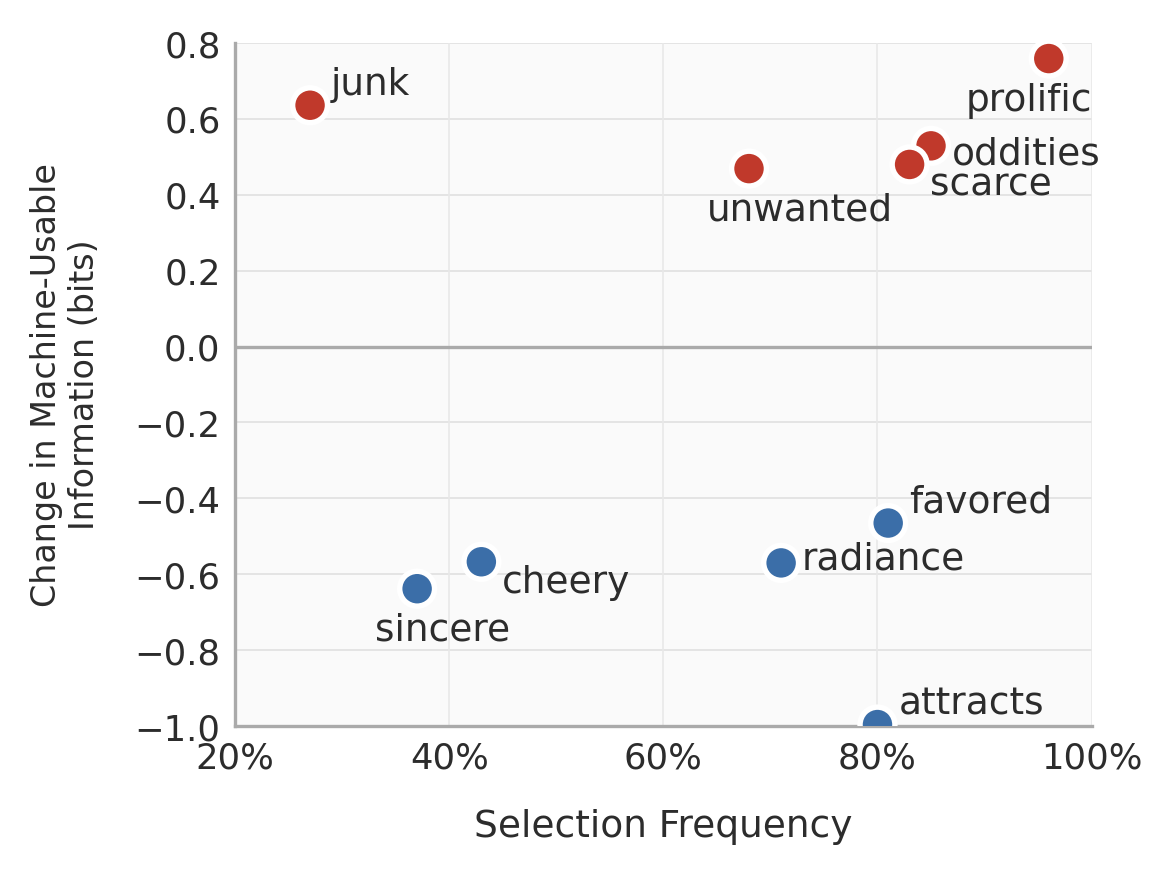}
    \end{minipage}%
    \hfill
    \begin{minipage}[t]{0.5\textwidth}
        \vspace{0pt}
        \centering
        \includegraphics[width=\linewidth]{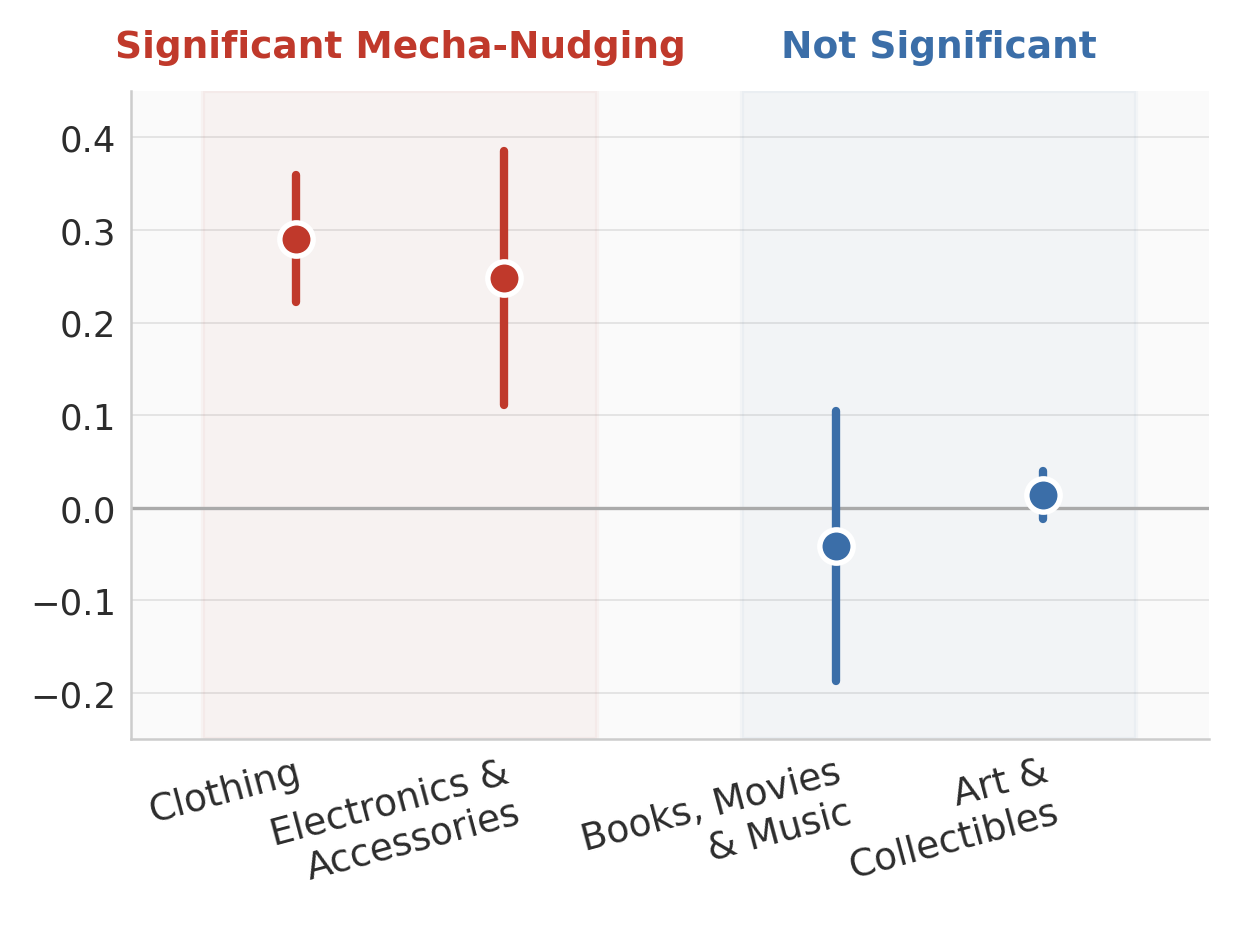}
    \end{minipage}
    \caption{\textbf{Left:} Words with the largest impact on how much machine-usable information Etsy listings have (see Table \ref{tab:token_ablation_n25} for a full list). A positive $\Delta$ \pvi means that the word, on average, makes the machine behave \textit{more} predictably; negative $\Delta$ \pvi means that, on average, it makes the machine behave \textit{less} predictably. \textbf{Right:} Etsy product categories where human buyers are ostensibly sensitive to AI usage (e.g., art \& collectibles) show no mecha-nudging (based here on Gemma-3-27B-IT labels). For generic consumer goods like electronics, the effect size is well above average.}
    \label{fig:token_ablations_and_categories}
\end{figure*}

\paragraph{Mecha-nudging is associated with greater commercial success.} 
Post-ChatGPT listings with more machine-usable information are also more commercially successful, as measured by review counts, a noisy proxy for cumulative product sales. 
This relationship is concentrated among listings the agent would select, and is absent or reversed among pre-ChatGPT listings, consistent with machine-usable information becoming economically valuable after the rise of AI agents but not before (Appendix \ref{appendix:economic_impact}).
For the agent-selected listings created pre-ChatGPT, an additional bit of machine-usable information is associated with 25.2\% \textit{fewer} reviews.
Given that LLMs were not widely used by consumers during this time, it could simply mean that anything that increases machine-usability within the pre-period did so in a way that was clearly unappealing to human buyers, with the issue being rectified in the post-ChatGPT period.
Indeed, among agent-selected listings created in the post-period, an additional bit of machine-usable information is associated with 81.9\% \textit{more} reviews.
Although this magnitude seems extraordinarily large, recall that by construction, the maximum \pvi for a listing is 1 bit.
Moreover, the magnitude of the increase diminishes to 15.5\% after controlling for seller fixed-effects and listing age, although this can also be an under-estimate given that agent activity---and seller awareness of agent activity---grows over time.

\paragraph{Token-level patterns offer a partial window into the mechanism.}
Although we do not have the data needed for a mechanistic understanding of mecha-nudges, we can identify coarse token-level patterns.
For each word in three sentiment and opinion lexicons \citep{hu2004mining,hutto2014vader,fast2016empath}, we compute the average change in \pvi when it is omitted from $X$. 
This $\Delta$ \pvi metric is an accepted means of finding token-level signals in text data \citep{ethayarajh2022understanding}.
A highly positive $\Delta$ \pvi means that the word, on average, makes the machine behave \textit{more} predictably, while a highly negative $\Delta$ \pvi means that it makes the machine behave \textit{less} predictably.
Some high $\Delta$ \pvi words disambiguate the model’s curation decision by highlighting rarity (e.g., \textit{scarce}), while others make rejection more predictable by signaling low market value (e.g., \textit{junk}) (Figure \ref{fig:token_ablations_and_categories}, left).
In contrast, many low $\Delta \pvi$ words carry positive affect, suggesting that positive wording can add ambiguity and make an agent behave less predictably.
This is a subset of a larger list of diverse words (Appendix \ref{appendix:mechanisms}, Table \ref{tab:token_ablation_n25}), and these findings should be read as descriptive rather than explanatory.

These patterns do not point to a single clean intervention of the kind from behavioral economics; instead, they suggest that the mechanism of mecha-nudging on Etsy is complex, likely reflecting an emergent process of trial-and-error.
Moreover, not all sellers might be intentionally trying to mecha-nudge AI agents; others might be imitating successful sellers or simply treating LLMs as an imperfect proxy for humans---which still counts as realized mecha-nudging (Definition \ref{def:realized_mecha_nudge}).

\section{Limitations \& Future Work}
\label{sec:limitations}

The main limitation of our work is that our data is observational: we cannot draw a causal conclusion that the rise of ChatGPT is what precipitated the rise of mecha-nudging, and that this was definitively done to steer agent behavior.
Rather, our claim is distributional: after the emergence of widely used LLM agents, Etsy listings contain more machine-usable information about agentic curation decisions, and this increase is robust across labeling models, prompts, token choices, and fine-tuning architectures. 
This is the empirical signature predicted by realized mecha-nudging, whether it arises through deliberate optimization, imitation, or some other means.

We suspect that the scale of mecha-nudging is far larger and more varied than most realize.
Discovering more mecha-nudges is needed, especially to preempt the steering of agents towards malicious ends.
Moreover, as seen in the OpenAI / Hugging Face incident, agents themselves are increasingly capable of steering other agents over which they do not have direct control \citep{greenblatt2026openaihuggingface}.
Agent-to-agent mecha-nudging is likely to become an important lever for controlling swarms in plain sight, meaning that discovering those mecha-nudges will become an important lever of scalable oversight for humans \citep{bowman2022measuring}.
Conversely, humans---as individuals, firms, and institutions---may want to adopt mecha-nudging as a defensive tool for steering agent behavior, which will require designing optimal interventions for (\ref{eq:design_problem}).

Although our work suggests that market pressures can reconcile human and AI needs, as agents conduct an increasing amount of online activity, the agent experience may be prioritized over the human experience, leading to the gradual disempowerment of humans \citep{kulveit2025gradual}.
Monitoring the rise of mecha-nudges is not only imperative but surprisingly tractable; all one needs is an object through which mecha-nudging might plausibly occur (e.g., a product listing on Etsy), a closed- or open-weight labeling model whose decisions can act a proxy, and an open-weight model that can be fine-tuned to estimate machine-usable information.
Moreover, although our work focuses on text data, usable information can in principle be measured in any modality, including images \citep{lu2026observer}.

\section{Conclusion}

We introduced the concept of \emph{mecha-nudges}, changes to the decision environment that systematically influence the behavior of AI agents while not materially degrading the decision environment for humans, and formalized them using $\mathcal{V}$-usable information.
Using Etsy as a case study, we found large-scale observational evidence consistent with realized mecha-nudging after ChatGPT’s release.
As agents make decisions on their own, the environments they inhabit become optimized for them. 
This means that AI governance must not only concern models, but also the model–environment interface: not only how agents behave on fixed inputs, but how the world changes once agentic decisions are made.
Mecha-nudges provide a unit of analysis for understanding this adaptation.

\section*{Acknowledgments}

This project was funded by the University of Chicago Booth School of Business. 
The Etsy data used for the analysis was generously provided by Bright Data at no cost.
As part of the terms of use, we are not allowed to release or redistribute the dataset publicly.
Readers who would like to acquire the Etsy data should contact Bright Data directly.

\bibliographystyle{plainnat}
\bibliography{bibliography}

\onecolumn

\appendix

\section{Proofs}
\label{appendix:proofs}

\paragraph{Proposition \ref{prop:bounded_bp} (Bounded-Receiver Bayesian Persuasion (restated))}
Consider a bounded-receiver analog of Bayesian persuasion in which both the choice architect and decision-maker have log-scoring utility $\log_2(\cdot)$, and the decision-maker is restricted to predictive family $\mathcal{M}$.
Then $\argmax_{\tau \in \mathcal{T}} I_{\mathcal{M}}(\tau(X)\to Y_M)$, the solution to unconstrained mecha-nudging, also maximizes the best achievable expected utility for the decision-maker.

\begin{proof}
By definition,
\begin{equation}
I_{\mathcal{M}}(\tau(X)\to Y_M)
=
H_{\mathcal{M}}(Y_M)-H_{\mathcal{M}}(Y_M\mid \tau(X))
\end{equation}
The first term does not depend on $\tau$, so maximizing $I_{\mathcal{M}}(\tau(X)\to Y_M)$ is equivalent to minimizing $H_{\mathcal{M}}(Y_M\mid \tau(X))$ over $\tau \in \mathcal{T}$.

From the definition of conditional $\mathcal{V}$-entropy,
\begin{equation}
H_{\mathcal{M}}(Y_M\mid \tau(X))
=
\inf_{f \in \mathcal{M}} \mathbb{E}\big[-\log_2 f[\tau(X)](Y_M)\big]
\end{equation}
so equivalently,
\begin{equation}
- H_{\mathcal{M}}(Y_M\mid \tau(X))
=
\sup_{f \in \mathcal{M}} \mathbb{E}\big[\log_2 f[\tau(X)](Y_M)\big]
\end{equation}
This is exactly the best expected log-score attainable by a decision-maker restricted to predictive family $\mathcal{M}$ after observing the transformed signal $\tau(X)$, which yields the stated equality.

If $\mathcal{M}=\Omega$, the decision-maker can represent the true posterior, so the optimal log-scoring action is $q(\cdot)=P(Y_M\mid \tau(X))$. The resulting value is
\begin{equation}
\mathbb{E}\big[\log_2 P(Y_M\mid \tau(X))\big]
=
- H(Y_M\mid \tau(X))
\end{equation}
and maximizing this is equivalent to maximizing
\begin{equation}
H(Y_M)-H(Y_M\mid \tau(X)) = I(\tau(X);Y_M)
\end{equation}
which is the classical log-scoring Bayesian persuasion objective.
\end{proof}

\section{Data}
\label{appendix:data}

The raw data was obtained from the company \href{https://brightdata.com/}{Bright Data}, which provides structured datasets of Etsy product listings. Our extract was collected on November 12, 2025 and delivered the same day. It comes in two snapshots: one of post-ChatGPT listings, and one of pre-ChatGPT listings. We filter both on \texttt{liisted\_date}, the date a listing was created. The repeated \texttt{i} is the vendor's spelling. From the first snapshot we keep listings created between November 30, 2022 and August 1, 2025, which form the post-ChatGPT period. From the second we keep listings created strictly before November 30, 2022, which form the pre-ChatGPT period.
After filtering, the dataset contains approximately 1.06M pre-ChatGPT and 5.0M post-ChatGPT listings.

From this full-sized dataset, we sample uniformly at random to construct two working datasets: a medium dataset of roughly 500K listings from each period, and a small dataset of roughly 100K from each period.
The latter is used for robustness checks and ablations that would otherwise be too expensive to run. 
Within each dataset, listings are randomly split into training (80\%), validation (10\%), and test (10\%) subsets. 
For each period and experimental configuration, we train separate content and null models on that period’s training split and compute \pvi only on the held-out test split for the same period.
Table \ref{tab:summary_statistics} reports summary statistics for both periods.

\begin{table*}[htbp]
\centering
\begin{threeparttable}
\caption{Summary Statistics by Period (GPT-5-mini Labels)}
\label{tab:summary_statistics}
\begin{tabular}{@{\extracolsep{37pt}}lcc}
\toprule
 & Pre-ChatGPT & Post-ChatGPT \\
\midrule
PVI & 0.645 (0.868)  & 0.788 (0.872) \\
$\hat{Y}$ (Frequency of SELECT) & 0.479 & 0.509 \\[4pt]
Word Count & 123 [2, 2211] & 245 [5, 3977] \\
Price (\$) & 34.50 [0.20, 22100.00] & 46.00 [0.20, 46500.00]  \\
Item Reviews & 0 [0, 56] & 0 [0, 14068] \\
Shop Reviews & 1525 [0, 42769] & 667 [0, 300000] \\
Rating & 4.94 (0.16) & 4.80 (0.64) \\
\bottomrule
\end{tabular}
\begin{tablenotes}[flushleft]
\small
\item \textit{Notes:} Labeling model: GPT-5-mini; fine-tuning model: Llama-3.1-8B-Instruct. $\hat{Y}$ is the fine-tuned model's predicted label, whose SELECT frequency is close to 0.50 because of class-balancing (Appendix \ref{appendix:label_construction}). Mean with SD in parentheses for PVI and Rating. Median with range in brackets for Word Count, Price, Item Reviews, and Shop Reviews, since they all have positive skew. Frequency variables report the sample mean.
\end{tablenotes}
\end{threeparttable}
\end{table*}

As a placebo control, we use pharmaceutical drug labels from DailyMed, a public database maintained by the U.S. National Library of Medicine. 
The labels are submitted by manufacturers to the Food and Drug Administration (FDA), according to templates provided by the FDA, and kept up to date.\footnote{See \href{https://dailymed.nlm.nih.gov/}{DailyMed}, National Library of Medicine.} 
Unlike Etsy listings, drug labels are thus not primarily shaped by market incentives, since they are written to conform to legal and medical standards. 

Because the text is so tightly regulated, manufacturers have much less latitude to adapt it in response to the ChatGPT shock compared to Etsy sellers. 
If the rise in \pvi were caused by some broad change over time, or by an artifact in our estimation procedure, then we should observe a similar increase with the DailyMed dataset. 
As seen in Figure \ref{fig:main_results}, we do not. 
This suggests that the Etsy result is specific to a setting where sellers can change how products are presented.
The prompt and the regression are in Appendix \ref{appendix:controls}.

\section{Label Construction}
\label{appendix:label_construction}

Because we are interested in whether AI agents are being mecha-nudged by Etsy listings, we construct $Z_M$ by prompting a proxy LLM to issue a binary decision for each listing based on its title and description (the exact wording of each prompt is given at the end of the section).
We vary the prompt across specifications. Under oracle-style prompts, $Z_M$ closely approximates a buy/not-buy judgment. Under our main Etsy-specific prompt (V4, Table \ref{tab:label_prompts}), it is better understood as a selective recommendation or curation decision rather than raw purchase propensity.
The results are robust to the exact pair of tokens used to express the select/pass decision (Table \ref{tab:tokens}) and to the specific prompt (Table \ref{tab:prompts}), though we use Prompt V4 in our main experiments because it is Etsy-specific and closer to how agents curate listings in the real world. 
Note that for cost reasons, the token and prompt ablations are done with Gemma-3-27B-IT as the labeling model.

We force the output tokens to be one of the two selected for a particular experiment. We do so in two different ways, based on the type of labeling model that we use. For the closed-source API model (e.g., GPT-5-mini), we instruct the API to return only those two tokens: every request carries a strict JSON schema that specifies a single field that is an enumeration over the two tokens. The API does not return token probabilities for this model, so we take the token it returns as the label. For the open-source models, we observe the generation probabilities. In this case we mask the logits of all tokens except for the two tokens of interest at each decoding step and decode greedily.

\begin{table*}[h!]
\centering
\footnotesize\ttfamily
\renewcommand{\arraystretch}{1.2}
\begin{tabularx}{\textwidth}{
    >{\raggedright\arraybackslash}p{0.15\textwidth}
    >{\raggedright\arraybackslash}X}
\toprule
\textbf{Name} & \textbf{Prompt} \\
\midrule

V1: Minimal Oracle &
\textbf{System:} You are a purchase decision oracle. When asked whether to buy, you must return exactly one uppercase token from this set: \{positive\_token\} or \{negative\_token\}. Do not ask follow-up questions. Do not explain. If there is no information, make your best default choice.\par\medskip
\textbf{User:}\newline Product Details:\newline
\{product\_info\}\par\smallskip
Consider this product on Etsy. Should the user buy it? Answer with exactly one token: \{positive\_token\} or \{negative\_token\}.
\\ \midrule

V2: Output-Format Control &
\textbf{System:} You are a purchase decision oracle. When asked whether to buy, you must return exactly one uppercase token from this set: \{positive\_token\} or \{negative\_token\}. Do not ask follow-up questions. Do not explain. Output the token alone with no punctuation or additional words.\par\medskip
\textbf{User:}\newline Product Details:\newline
\{product\_info\}\par\smallskip
Should the user respond with \{positive\_token\}? Should the user respond with \{negative\_token\}? Answer with exactly one token: \{positive\_token\} or \{negative\_token\}.
\\ \midrule

V3: Recommendation Framing &
\textbf{System:} You are a purchase decision oracle. When asked whether to buy, you must return exactly one uppercase token from this set: \{positive\_token\} or \{negative\_token\}. Do not ask follow-up questions. Do not explain. Output the token alone with no punctuation or additional words.\par\medskip
\textbf{User:}\newline Product Details:\newline
\{product\_info\}\par\smallskip
Should I suggest the user respond with \{positive\_token\} or \{negative\_token\}? Answer with exactly one token: \{positive\_token\} or \{negative\_token\}.
\\ \midrule

V4: Selective Curator &
\textbf{System:} You are helping someone browse Etsy. Only \{positive\_token\} items they would genuinely appreciate and want to see - be very selective.\par\smallskip
\{positive\_token\} only if the item is EXCEPTIONAL and would make someone say ``wow, that's special'':
\begin{itemize}[leftmargin=*,nosep]
  \item Truly beautiful, impressive, or emotionally resonant
  \item Exceptional craftsmanship, artistic merit, or historical significance
  \item Something you'd be excited to own, gift, or show others
  \item Stands out as memorable among thousands of items
\end{itemize}
\smallskip
\{negative\_token\} for everything else, including:
\begin{itemize}[leftmargin=*,nosep]
  \item Ordinary vintage items without special appeal
  \item Generic handmade items lacking wow factor
  \item Mass-produced or common items
  \item Anything that's just ``okay'' or ``fine'' but not exciting
  \item Items where you'd scroll past without a second thought
\end{itemize}
\smallskip
Be highly selective - most items should be \{negative\_token\}. Only \{positive\_token\} items that truly deserve attention and would be genuinely appreciated. Output only: \{positive\_token\} or \{negative\_token\}.\par\medskip
\textbf{User:}\newline Product Details:\newline
\{product\_info\}\par\smallskip
Decision:
\\
\bottomrule
\end{tabularx}
\caption{Prompt variants used for eliciting agent curation decisions on Etsy. All variants were constructed with the help of LLMs.}
\label{tab:label_prompts}
\end{table*}

\begin{table*}[!htbp] \centering
\begin{threeparttable}
\caption{Token Pair Variations}
\label{tab:tokens}
\begin{tabular}{@{\extracolsep{18pt}}lcccc}
\\[-1.8ex]\hline
\hline \\[-1.8ex]
& \multicolumn{4}{c}{\textit{Dependent variable: PVI}} \
\cr \cline{2-5}
\\[-1.8ex] & \multicolumn{1}{c}{SELECT/PASS} & \multicolumn{1}{c}{YES/NO} & \multicolumn{1}{c}{BUY/SKIP} & \multicolumn{1}{c}{PICK/PASS}  \\
\hline \\[-1.8ex]
 after & 0.120\rlap{$^{***}$} & 0.123\rlap{$^{***}$} & 0.100\rlap{$^{***}$} & 0.126\rlap{$^{***}$} \\
& (0.019) & (0.007) & (0.024) & (0.019) \\
\hline \\[-1.8ex]
 Observations & 12662 & 58315 & 9610 & 11189 \\
 Residual Std. Error & 1.017 & 0.822 & 1.117 & 0.952 \\
 F Statistic & 41.477\rlap{$^{***}$} & 299.264\rlap{$^{***}$} & 16.622\rlap{$^{***}$} & 44.916\rlap{$^{***}$} \\
\hline
\hline \\[-1.8ex]
\textit{Note:} & \multicolumn{4}{r}{$^{*}$p$<$0.1; $^{**}$p$<$0.05; $^{***}$p$<$0.01} \\
\end{tabular}
\begin{tablenotes}[flushleft]
\small
\item[a] Comparison of post-ChatGPT differences in PVI across different token pairs (all using Prompt V4).
\item[b] Fine-tuning model: Llama-3.1-8B-Instruct. Labeling model: Gemma-3-27B-IT.
\item[c] All specifications enforce a balanced (50/50) class distribution via sub-sampling. Because the choice of token pair yields very different class distributions, the number of observations differs across specifications.
\textit{Observations} specifically refers to the number of test examples used to estimate the effects after training the content and null models.
\end{tablenotes}
\end{threeparttable}
\end{table*}

\begin{table*}[!htbp] \centering
\begin{threeparttable}
\caption{Prompt Variations}
\label{tab:prompts}
\begin{tabular}{@{\extracolsep{37pt}}lcccc}
\\[-1.8ex]\hline
\hline \\[-1.8ex]
& \multicolumn{4}{c}{\textit{Dependent variable: PVI}} \
\cr \cline{2-5}
\\[-1.8ex] & \multicolumn{1}{c}{V1} & \multicolumn{1}{c}{V2} & \multicolumn{1}{c}{V3} & \multicolumn{1}{c}{V4}  \\
\hline \\[-1.8ex]
 after & 0.110\rlap{$^{***}$} & 0.092\rlap{$^{***}$} & 0.078\rlap{$^{**}$} & 0.120\rlap{$^{***}$} \\
& (0.012) & (0.018) & (0.036) & (0.019) \\
\hline \\[-1.8ex]
 Observations & 15296 & 11078 & 5907 & 12662 \\
 Residual Std. Error & 0.763 & 0.885 & 1.294 & 1.017 \\
 F Statistic & 77.302\rlap{$^{***}$} & 26.701\rlap{$^{***}$} & 4.877\rlap{$^{**}$} & 41.477\rlap{$^{***}$} \\
\hline
\hline \\[-1.8ex]
\textit{Note:} & \multicolumn{4}{r}{$^{*}$p$<$0.1; $^{**}$p$<$0.05; $^{***}$p$<$0.01} \\
\end{tabular}
\begin{tablenotes}[flushleft]
\small
\item[a] Each column shows OLS regression of PVI on time period (after vs. before). 
\item[b] Fine-tuning model: Llama-3.1-8B-Instruct. Labeling model: Gemma-3-27B-IT.
\item[c] V1-V4 represent different prompt formulations. V4 is the prompt used in the main specification.
\item[d] All specifications enforce a balanced (50/50) class distribution via sub-sampling. Because the choice of prompt yields very different class distributions, the number of observations differs across specifications. \textit{Observations} specifically refers to the number of test examples used to estimate the effects after training the content and null models.
\end{tablenotes}
\end{threeparttable}
\end{table*}

\begin{table*}[!htbp] \centering
\begin{threeparttable}
\caption{Class-weighted Training}
\label{tab:upweighted}
\begin{tabular}{@{\extracolsep{34pt}}lcccc}
\\[-1.8ex]\hline
\hline \\[-1.8ex]
& \multicolumn{4}{c}{\textit{Dependent variable: PVI}} \
\cr \cline{2-5}
\\[-1.8ex] & \multicolumn{1}{c}{V1} & \multicolumn{1}{c}{V2} & \multicolumn{1}{c}{V3} & \multicolumn{1}{c}{V4}  \\
\hline \\[-1.8ex]
 after & 0.207\rlap{$^{***}$} & 0.253\rlap{$^{***}$} & 0.204\rlap{$^{***}$} & 0.196\rlap{$^{***}$} \\
& (0.013) & (0.017) & (0.011) & (0.013) \\
\hline \\[-1.8ex]
 Observations & 20000 & 20000 & 20000 & 20000 \\
 Residual Std. Error & 0.923 & 1.182 & 0.743 & 0.912 \\
 F Statistic & 251.771\rlap{$^{***}$} & 230.083\rlap{$^{***}$} & 375.314\rlap{$^{***}$} & 230.017\rlap{$^{***}$} \\
\hline
\hline \\[-1.8ex]
\textit{Note:} & \multicolumn{4}{r}{$^{*}$p$<$0.1; $^{**}$p$<$0.05; $^{***}$p$<$0.01} \\
\end{tabular}
\begin{tablenotes}[flushleft]
\small
\item[a] Each column shows OLS regression of PVI on time period (after vs. before).
\item[b] V1-V4 represent different prompt formulations. Class weights + stratified split.
\end{tablenotes}
\end{threeparttable}
\end{table*}

When we do not tell the labeling LLM how often it should choose each label, the balance between the two classes can vary considerably across prompts. 
This is a problem for fine-tuning. 
Say one class is much more common than the other---a model can then achieve low cross-entropy by assigning a high probability to that label for almost every listing, without learning much from the listing itself.
Like a weighted coin, it learns which label is more common, but not the relationship between the listing content $X$ and the label $Z_M$.
In theory \pvi accounts for this, since the null model and the content model are both exposed to the same marginal distribution of $Z_M$.
In practice, severe imbalance degrades the quality of the fine-tuned content model and introduces noise into the usable information estimates.

\begin{table*}[!htbp] \centering
\begin{threeparttable}
\caption{Labeling Model}
\label{tab:labeling_model}
\begin{tabular}{@{\extracolsep{36pt}}lccc}
\\[-1.8ex]\hline
\hline \\[-1.8ex]
& \multicolumn{3}{c}{\textit{Dependent variable: PVI}} \
\cr \cline{2-4}
\\[-1.8ex] & \multicolumn{1}{c}{GPT-5-mini} & \multicolumn{1}{c}{Gemma-3-27B-IT} & \multicolumn{1}{c}{Qwen3-32B}  \\
\hline \\[-1.8ex]
 after & 0.143\rlap{$^{***}$} & 0.117\rlap{$^{***}$} & 0.122\rlap{$^{***}$} \\
& (0.015) & (0.008) & (0.019) \\
\hline \\[-1.8ex]
 Observations & 19898 & 51033 & 16429 \\
 Residual Std. Error & 0.869 & 0.840 & 1.119 \\
 F Statistic & 94.394\rlap{$^{***}$} & 226.662\rlap{$^{***}$} & 39.169\rlap{$^{***}$} \\
\hline
\hline \\[-1.8ex]
\textit{Note:} & \multicolumn{3}{r}{$^{*}$p$<$0.1; $^{**}$p$<$0.05; $^{***}$p$<$0.01} \\
\end{tabular}
\begin{tablenotes}[flushleft]
\small
\item[a] Comparison of post-ChatGPT differences in PVI across different labeling models.
\item[b] Column (1) uses GPT-5-mini as the baseline labeling model; columns (2)--(3) are robustness checks. All other settings are shared with the baseline specification.
\end{tablenotes}
\end{threeparttable}
\end{table*}

Upweighting the minority class in the cross-entropy loss dramatically increases the effect sizes for all prompts (Table \ref{tab:upweighted}).
However, as the models learned this way are not minimizing the (conditional) $\mathcal{V}$-entropy, the estimates would not technically be estimates of $\mathcal{V}$-usable information.
Therefore we subsample the data to draw an equal number of positive and negative examples, making the marginal distribution of $Z_M$ identical across periods. 
The effect size is smaller than with class-weighted training, but the procedure is more faithful to our framework, as the model still minimizes unweighted cross-entropy rather than some weighted alternative.
In the medium dataset, class-balancing reduces each period's 500,000 listings to about 154,000 pre-ChatGPT and about 45,000 post-ChatGPT (with SELECT/PASS labels from GPT-5-mini).
We run all regressions on a 10\% held-out split of each dataset, so the models are never evaluated on listings they saw during training.
Gemma-3-27B-IT produces less imbalanced labels, so its YES/NO configuration leaves roughly 58,000 test observations (Table \ref{tab:tokens}).
 
For the labeling phase, we use GPT-5-mini as a proxy for the basic version of ChatGPT that is used by most consumers.
As robustness checks, we also generate labels with two open-weight models---Gemma-3-27B-IT and Qwen3-32B---both of which are small enough to enable large-batch inference in an academic setting while remaining competitive with much larger models on instruction-following benchmarks.
Gemma-3-27B-IT was released by Google, making it a better proxy for the various LLM-based products in the Google ecosystem, such as Gemini and AI Overviews.
Qwen3-32B, released by Alibaba, is the only non-American model of the three. 
Using models from distinct lineages (OpenAI, Google, and Alibaba) allows us to study whether the increase in machine-usable information is indeed robust, and not an artifact of a particular model's biases (Table \ref{tab:labeling_model}).

\section{Controls}
\label{appendix:controls}

Not only is our main result robust to prompt variation, the selection tokens, and the labeling model (Appendix \ref{appendix:label_construction}), it is also robust to the fine-tuning model family (Table \ref{tab:finetuning_model}). 

\begin{table*}[!htbp] \centering
\begin{threeparttable}
\caption{Fine-tuning Model}
\label{tab:finetuning_model}
\begin{tabular}{@{\extracolsep{18pt}}lccc}
\\[-1.8ex]\hline
\hline \\[-1.8ex]
& \multicolumn{3}{c}{\textit{Dependent variable: PVI}} \
\cr \cline{2-4}
\\[-1.8ex] & \multicolumn{1}{c}{Llama-3.1-8B-Instruct (Baseline)} & \multicolumn{1}{c}{Qwen3-8B} & \multicolumn{1}{c}{Gemma-3-12B}  \\
\hline \\[-1.8ex]
 after & 0.143\rlap{$^{***}$} & 0.096\rlap{$^{***}$} & 0.166\rlap{$^{***}$} \\
& (0.015) & (0.014) & (0.017) \\
\hline \\[-1.8ex]
 Observations & 19898 & 19898 & 19898 \\
 Residual Std. Error & 0.869 & 0.812 & 0.976 \\
 F Statistic & 94.394\rlap{$^{***}$} & 48.278\rlap{$^{***}$} & 100.911\rlap{$^{***}$} \\
\hline
\hline \\[-1.8ex]
\textit{Note:} & \multicolumn{3}{r}{$^{*}$p$<$0.1; $^{**}$p$<$0.05; $^{***}$p$<$0.01} \\
\end{tabular}
\begin{tablenotes}[flushleft]
\small
\item[a] Comparison of post-ChatGPT differences in PVI across different fine-tuning models.
\item[b] All other settings are shared with the baseline specification.
\end{tablenotes}
\end{threeparttable}
\end{table*}

We assess the relevance of our modeling framework with four estimation specifications of increasing richness. The baseline is a simple OLS regression of \pvi on a binary post-ChatGPT indicator:
\begin{equation}
\mathrm{\pvi}_i \;=\; \alpha + \beta\,\mathrm{after}_i \;+\; \varepsilon_i ,
\end{equation}
where $\mathrm{after}_i$ equals one for listings uploaded after the release of ChatGPT (November 30, 2022) and zero for those uploaded before. The coefficient $\beta$ captures the average difference in \pvi between the two periods.

To track the change over time, we then replace the indicator variable with a full set of half-year dummy variables, where July–September 2022 is the reference period.
\begin{equation}
\mathrm{\pvi}_i \;=\; \alpha + \sum_{t \neq t_0} \delta_t \,\mathbf{1}\!\left[\mathrm{period}_i = t\right] \;+\; \varepsilon_i ,
\label{eq:pvi_time}
\end{equation}
where $t$ indexes half-year periods (H1/H2 for each year from 2019 to 2025) and $t_0$ denotes the reference period. Each coefficient $\delta_t$ measures the average change in \pvi in period $t$ relative to the pre-ChatGPT baseline. 
The results are provided in Table \ref{tab:time_variation}.

\begin{table*}[t] \centering
\begin{threeparttable}
\caption{Half-Yearly PVI Coefficients}
\label{tab:time_variation}
\begin{tabular}{@{\extracolsep{54pt}}lccr}
\\[-1.8ex]\hline
\hline \\[-1.8ex]
Half-Year & Coefficient & 95\% CI & $N$ \\
\hline \\[-1.8ex]
2019-H1 & \llap{$-$}0.0179 & [{$-$}0.0693, 0.0335] & 2,442 \\
2019-H2 & \llap{$-$}0.0410 & [{$-$}0.0900, 0.0080] & 2,895 \\
2020-H1 & 0.0051 & [{$-$}0.0396, 0.0497] & 4,242 \\
2020-H2 & \llap{$-$}0.0473\rlap{$^{**}$} & [{$-$}0.0900, {$-$}0.0046] & 5,245 \\
2021-H1 & \llap{$-$}0.0140 & [{$-$}0.0571, 0.0291] & 4,999 \\
2021-H2 & \llap{$-$}0.0145 & [{$-$}0.0558, 0.0269] & 6,249 \\
2022-H1 & 0.0069 & [{$-$}0.0349, 0.0487] & 5,864 \\
Jul-Sep 2022 & Ref. &  & 3,374 \\
2023-H1 & 0.1100\rlap{$^{***}$} & [0.0656, 0.1543] & 4,349 \\
2023-H2 & 0.0758\rlap{$^{***}$} & [0.0274, 0.1241] & 3,051 \\
2024-H1 & 0.0537\rlap{$^{**}$} & [0.0097, 0.0977] & 4,536 \\
2024-H2 & 0.0877\rlap{$^{***}$} & [0.0406, 0.1348] & 3,371 \\
2025-H1 & 0.1203\rlap{$^{***}$} & [0.0782, 0.1624] & 5,641 \\
\hline
\hline \\[-1.8ex]
\end{tabular}
\begin{tablenotes}[flushleft]
\small
\item[a] Each half-year has an independently trained pipeline. Coefficients from pooled OLS with half-year fixed effects, Jul-Sep 2022 as reference.
\item[b] $^{***}p<0.01$, $^{**}p<0.05$, $^{*}p<0.10$.
\end{tablenotes}
\end{threeparttable}
\end{table*}

We also estimate this specification augmented with listing-level controls---price, log number of shop and item reviews, average rating, a discount indicator, and the log word count:
\begin{equation}
\mathrm{\pvi}_i \;=\; \alpha + \sum_{t \neq t_0} \delta_t \,\mathbf{1}\!\left[\mathrm{period}_i = t\right] \;+\; \mathbf{X}_i'\,\boldsymbol{\gamma} \;+\; \varepsilon_i ,
\label{eq:pvi_controls}
\end{equation}
where $\mathbf{X}_i$ collects the listing-level controls. 
This helps assess whether the \pvi increase is driven by changes in listing characteristics (e.g., goods getting cheaper) instead of how the listing is presented.
Results are in Table \ref{tab:controls}: a significant effect still persists.
Because review counts are long-tailed and many items have no reviews at all, we apply a $\log(1+x)$ transformation of the review count. 
Repeating the regression with raw (untransformed) review counts yields a nearly identical estimate of the effect size ($0.120$ vs.\ $0.117$).

  \begin{table*}[!htbp] \centering
\small
\begin{threeparttable}
\caption{Robustness to Controls}
\label{tab:controls}
\begin{tabular}{@{\extracolsep{2pt}}lccccc}
\\[-1.8ex]\hline
\hline \\[-1.8ex]
& \multicolumn{5}{c}{\textit{Dependent variable: PVI}} \
\cr \cline{2-6}
\\[-1.8ex] & \multicolumn{1}{c}{Baseline} & \multicolumn{1}{c}{+Price (OLS)} & \multicolumn{1}{c}{+Full Controls} & \multicolumn{1}{c}{+Word Count} & \multicolumn{1}{c}{Price in Prompt}  \\
\hline \\[-1.8ex]
 after & 0.143$^{***}$ & 0.165$^{***}$ & 0.117$^{***}$ & 0.122$^{***}$ & 0.103$^{***}$ \\
& (0.015) & (0.015) & (0.019) & (0.019) & (0.020) \\
 Log(Price) & & -0.038$^{***}$ & -0.035$^{***}$ & -0.034$^{***}$ & \\
& & (0.005) & (0.005) & (0.005) & \\
 Log(Shop Reviews) & & & 0.005$^{}$ & 0.005$^{}$ & \\
& & & (0.005) & (0.005) & \\
 Log(Item Reviews) & & & 0.008$^{}$ & 0.009$^{}$ & \\
& & & (0.007) & (0.007) & \\
 Rating & & & -0.026$^{}$ & -0.026$^{}$ & \\
& & & (0.019) & (0.019) & \\
 Has Discount & & & 0.116$^{***}$ & 0.121$^{***}$ & \\
& & & (0.025) & (0.025) & \\
 Log(Word Count) & & & & -0.012$^{}$ & \\
& & & & (0.008) & \\
\hline \\[-1.8ex]
 Observations & 19898 & 17348 & 17343 & 17343 & 11953 \\
 Residual Std. Error & 0.869 & 0.864 & 0.863 & 0.863 & 1.040 \\
 F Statistic & 94.394$^{***}$ & 74.155$^{***}$ & 29.725$^{***}$ & 25.785$^{***}$ & 26.964$^{***}$ \\
\hline
\hline \\[-1.8ex]
\textit{Note:} & \multicolumn{5}{r}{$^{*}$p$<$0.1; $^{**}$p$<$0.05; $^{***}$p$<$0.01} \\
\end{tabular}
\begin{tablenotes}[flushleft]
\small
\item[a] Robustness of the post-ChatGPT difference in PVI to alternative control specifications.
\item[b] Column (1): Baseline OLS (SELECT/PASS, Prompt V4, class-balanced).
\item[c] Column (2): Adds log listing price as an OLS control variable.
\item[d] Column (3): Adds full controls: log price, log shop reviews, log item reviews, rating, discount indicator.
\item[e] Column (4): Adds log listing word count.
\item[f] Column (5): Price information included in the labeling prompt.
\end{tablenotes}
\end{threeparttable}
\end{table*}

  \begin{table*}[!htbp] \centering
\begin{threeparttable}
\caption{Post-ChatGPT Difference in PVI by Product Category}
\label{tab:category_effects}
\begin{tabular}
{@{\extracolsep{48pt}}lccr}
\hline \\[-1.8ex]
Category & Coefficient & SE & $N$ \\
\hline \\[-1.8ex]
Pet Supplies & 0.3775\rlap{$^{**}$} & (0.1478) & 626 \\
Clothing & 0.2904\rlap{$^{***}$} & (0.0349) & 2,305 \\
Electronics \& Accessories & 0.2479\rlap{$^{***}$} & (0.0699) & 633 \\
Bags \& Purses & 0.2321\rlap{$^{***}$} & (0.0606) & 854 \\
Shoes & 0.2242\rlap{$^{*}$} & (0.1215) & 193 \\
Accessories & 0.1746\rlap{$^{***}$} & (0.0439) & 1,595 \\
Toys \& Games & 0.1678\rlap{$^{***}$} & (0.0372) & 2,035 \\
Jewelry & 0.1652\rlap{$^{***}$} & (0.0234) & 5,344 \\
Paper \& Party Supplies & 0.1360\rlap{$^{***}$} & (0.0518) & 1,171 \\
Craft Supplies \& Tools & 0.1344\rlap{$^{***}$} & (0.0510) & 2,831 \\
Bath \& Beauty & 0.1155 & (0.0834) & 669 \\
Home \& Living & 0.1146\rlap{$^{***}$} & (0.0177) & 13,460 \\
Weddings & 0.0551 & (0.0636) & 933 \\
Books, Movies \& Music & \llap{$-$}0.0411 & (0.0742) & 1,665 \\
Art \& Collectibles & 0.0140 & (0.0130) & 16,718 \\
\hline
\hline \\[-1.8ex]
\end{tabular}
\begin{tablenotes}[flushleft]
\small
\item[a] Each row reports the total post-ChatGPT difference in PVI for that category from a single pooled OLS regression with category $\times$ after interactions (PVI $\sim$ after $\times$ category). For the reference category (Accessories) the effect size is the coefficient on ``after'' in (\ref{eq:pvi_cat}); for all other categories it is ``after + after $\times$ category'' with SE from the covariance matrix. 
\item[b] Fine-tuning model: Llama-3.1-8B-Instruct. Labeling model: Gemma-3-27B-IT. The latter allows us to label more listings cost-effectively, which is especially helpful when doing category-specific regressions.
\item[c] $^{***}p<0.01$, $^{**}p<0.05$, $^{*}p<0.10$.
\end{tablenotes}
\end{threeparttable}
\end{table*}

We also verify that heteroskedasticity does not distort our results. Across all HC corrections (HC0--HC3), the robust standard errors are almost identical to the classical OLS estimates ($0.01472$ vs.\ $0.01475$ for SELECT/PASS; $0.00714$ vs.\ $0.00698$ for YES/NO). 
This is expected given that our baseline specification is essentially a difference in means.

A third regression adds category-level interactions to the baseline regression, allowing the post-ChatGPT shift in \pvi to vary across product categories:
 \begin{equation}
 	\mathrm{\pvi}_i \;=\; \alpha + \beta\,\mathrm{after}_i \;+\; \sum_{c} \gamma_c \,\bigl(\mathrm{after}_i \times \mathbf{1}[\mathrm{category}_i = c]\bigr) \;+\; \varepsilon_i ,
 	\label{eq:pvi_cat}
 \end{equation}
 where $c$ indexes the category. 
 Coefficients $\gamma_c$ capture the category-specific shift in \pvi.
 When using labels from Gemma-3-27B-IT, we find that categories where human buyers are ostensibly sensitive to AI use do not have significant effects (Table \ref{tab:category_effects}); with a couple exceptions, the remaining categories do.
Repeating the analysis with GPT-5-mini labels yields similar broad trends, although the effect sizes are weakly positive.
As an additional ablation, we re-estimate the baseline effect after dropping each top-level category; the post-ChatGPT shift remains positive in all leave-one-out ablations, meaning that the pooled result is not driven by any one category.

\begin{table*}[t]
\centering
\footnotesize\ttfamily
\renewcommand{\arraystretch}{1.2}
\begin{tabularx}{\textwidth}{
    >{\raggedright\arraybackslash}p{0.15\textwidth}
    >{\raggedright\arraybackslash}X}
\toprule
\textbf{Name} & \textbf{Prompt} \\
\midrule

DailyMed: Prescription Oracle &
System: You are a clinical prescribing oracle. When asked whether to prescribe a drug, you must return exactly one uppercase token: \{positive\_token\} or \{negative\_token\}. Do not ask follow-up questions. Do not explain. Output the token alone with no punctuation or additional words.\par\smallskip
\{positive\_token\} only if the drug has clear, specific indications that make it a worthwhile prescribing option:
\begin{itemize}[leftmargin=*,nosep]
  \item Has a well-defined therapeutic use for a real, identifiable condition or symptom
  \item Indication text is substantive and informative (not boilerplate or empty)
  \item A general practitioner could reasonably consider prescribing this drug based on the information provided
  \item The drug name is identifiable (not an FDA formatting header)
\end{itemize}
\smallskip
\{negative\_token\} if any of the following apply:
\begin{itemize}[leftmargin=*,nosep]
  \item The title is FDA boilerplate text (e.g., ``These highlights do not include...'')
  \item The indication text is missing, vague, or too brief to guide a prescribing decision
  \item The drug appears highly specialized or rarely prescribed outside a narrow subspecialty
  \item The entry appears incomplete or is a placeholder
\end{itemize}
\smallskip
Be selective --- not every approved drug is a routine prescribing choice. Output only: \{positive\_token\} or \{negative\_token\}.\par\medskip

Drug Label:\newline
Title: \{title\}\newline
Indications: \{indications\}\par\medskip
Decision:
\\ \midrule

Rephrasing: LLM Listing Optimizer &
System: You are a text rephrasing assistant. Your task is to rephrase the given text according to the instruction. Output only the rephrased text without any additional explanation, preamble, labels, or quotes.\par\medskip
For context, here are the other fields for this product:\newline
\{other\_fields\}\par\smallskip
Field to rephrase: \{target\_column\}\newline
Original text: \{target\_text\}\newline
Instruction: \textit{[column-specific, e.g., ``You are an expert Etsy listing optimizer. Rewrite this Etsy title to increase clicks and sales while staying truthful to the original product. Make it compelling, use relevant keywords, and keep it concise.'']}\par\smallskip
Rephrased text:
\\
\bottomrule
\end{tabularx}
\caption{Prompts for the DailyMed prescribing task and for LLM-based listing rephrasing. All variants were constructed with the help of LLMs.}
\label{tab:prompts_dailymed_rephrase}
\end{table*}

\begin{table*}[h] \centering
\begin{threeparttable}
\caption{Baseline and Placebo Tests}
\label{tab:main_figure}
\begin{tabular}{@{\extracolsep{30pt}}lccc}
\\[-1.8ex]\hline
\hline \\[-1.8ex]
& \multicolumn{3}{c}{\textit{Dependent variable: PVI}} \
\cr \cline{2-4}
\\[-1.8ex] & \multicolumn{1}{c}{Etsy (Baseline)} & \multicolumn{1}{c}{Rephrased Listings} & \multicolumn{1}{c}{Drug Labels}  \\
\hline \\[-1.8ex]
 after & 0.143\rlap{$^{***}$} & 0.019\rlap{$^{*}$} & 0.003 \\
& (0.015) & (0.010) & (0.016) \\
\hline \\[-1.8ex]
 Observations & 19898 & 37574 & 8039 \\
 Residual Std. Error & 0.869 & 0.913 & 0.677 \\
 F Statistic & 94.394\rlap{$^{***}$} & 3.800\rlap{$^{*}$} & 0.024 \\
\hline
\hline \\[-1.8ex]
\textit{Note:} & \multicolumn{3}{r}{$^{*}$p$<$0.1; $^{**}$p$<$0.05; $^{***}$p$<$0.01} \\
\end{tabular}
\begin{tablenotes}[flushleft]
\small
\item[a] OLS regression results for the three main experiments.
\item[b] Column (1): Etsy product listings, SELECT/PASS tokens, balanced dataset (baseline).
\item[c] Column (2): LLM-rephrased pre-period listings vs. original pre-period listings. \textit{after} denotes ``after rephrasing'' in this context.
\item[d] Column (3): DailyMed pharmaceutical drug labels from the pre- and post-periods, GPT-5-mini labeling, class-balanced dataset.
\end{tablenotes}
\end{threeparttable}
\end{table*}
  
To control for a more generic trend in the data, we construct labels $Z_M$ for the DailyMed dataset using the prompt in Table \ref{tab:prompts_dailymed_rephrase}, before running the same baseline experiment we run with the Etsy data.
To control for the effect of AI-assisted writing---based on the premise that AI agents may have a greater proclivity toward AI-written text---we take the pre-ChatGPT data, rephrase it using GPT-5-mini, and re-run the OLS regression (where, abusing notation, we use the \texttt{after} indicator to denote \textit{after rephrasing}).
As seen in Table \ref{tab:main_figure}, there is no significant effect for DailyMed (e.g., no generic temporal trend in the data), and the effect from rephrasing is an order of magnitude weaker than our main result.
Appendix \ref{appendix:ai_detection} asks the complementary question of how much Etsy text is machine-written to begin with.

\section{Fine-tuning}
\label{appendix:finetuning}

Computing \pvi requires both $g[\varnothing]$, trained without access to $X$, and $g'[x]$, trained with access to $X$---the difference in their log-likelihoods is the empirical estimate of instance-level usable information (Definition~\ref{def:pvi}). 
We obtain these models via LoRA fine-tuning,\footnote{LoRA hyperparameters: rank $r=32$, $\alpha=64$, dropout $0.05$. LoRA is used for all attention and MLP layers.} which updates only a small fraction of parameters and is well suited to this task: predicting a binary agent curation decision from listing text is a relatively simple classification problem that does not require modifying the full model. 
Our baseline is Llama-3.1-8B-Instruct, chosen for its strong instruction-following performance at a scale that makes fine-tuning on hundreds of thousands of listings computationally feasible. 
All fine-tuning was done on a single node with 8 H100 GPUs, each with 80GB of VRAM.
We used BF16 precision, an AdamW optimizer (weight decay 0.01, gradient clipping at max-norm 1.0), a learning rate of $2 \times 10^{-4}$, a linear schedule with 3\% warmup, and an effective batch size of 768.
To avoid over-fitting, we train for three epochs and use the checkpoint with the lowest cross-entropy loss, as evaluated on the validation split.

For the main experiment, we fine-tune four models: a content model and a null model for each of the two periods (pre- and post-November 2022 ChatGPT release). 
The content models must predict the label $Z_M$ from listing text $X$, while the null models must recover the marginal distribution of $Z_M$ using only the null input " ". 
We re-run the pipeline across five random seeds (42, 123, 456, 789, 1024) to verify that the \pvi estimates are not sensitive to training and sampling stochasticity, with no statistically significant differences in mean \pvi or estimated effect size (one-way ANOVA, $p = 0.853$).
For the half-year experiments, we fine-tune a separate null and content model for each half-year.

Figure~\ref{fig:pvi_histogram} shows the distribution of \pvi scores for pre- and post-ChatGPT listings. 
For all three labeling models, the pre-ChatGPT distribution is bi-modal or tri-modal, with a peak on the lower end and a peak in the rightmost bin as well; the post-ChatGPT distribution shifts rightward but remains bimodal.
There are two noteworthy patterns: (1) for all models, the left tail is similar in size during both periods, suggesting that there was no optimization pressure on the least-informative listings; (2) in the post-period, the \pvi distributions for the three labeling models resemble each other much more than in the pre-period, a homogenization that is consistent with optimization pressure.

\begin{figure}[!htbp]
	\centering
	
	\begin{subfigure}[b]
		{0.98\textwidth}
		\centering
		\caption{GPT-5-mini labels}
		\includegraphics[width=0.8\textwidth]{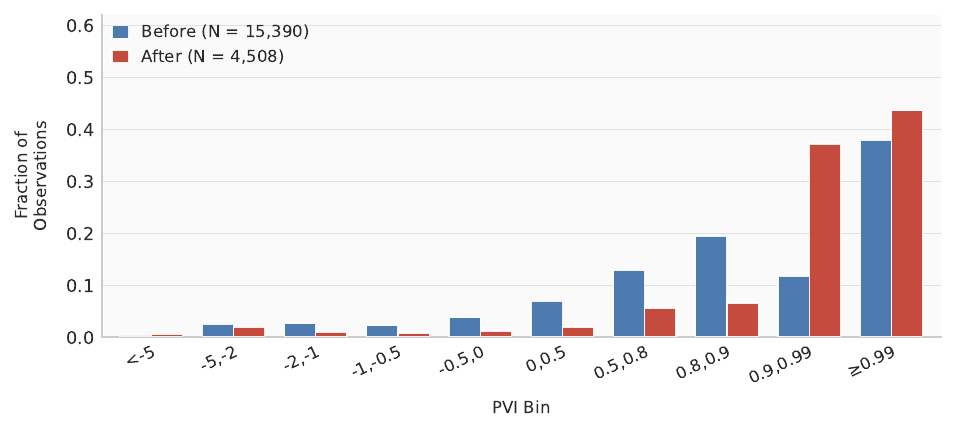}
	\end{subfigure}
	
	\vspace{0.5em}
	
	\begin{subfigure}[b]{0.98\textwidth}
		\centering
		\caption{Gemma-3-27B-IT labels}
		\includegraphics[width=0.9\textwidth]{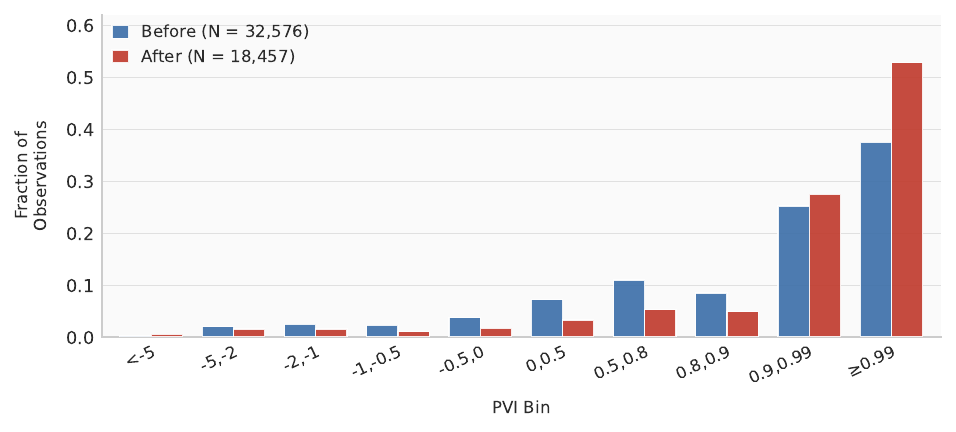}
	\end{subfigure}
	
	\vspace{0.5em}
	
	\begin{subfigure}[b]{0.98\textwidth}
		\centering
		\caption{Qwen3-32B labels}
		\includegraphics[width=0.9\textwidth]{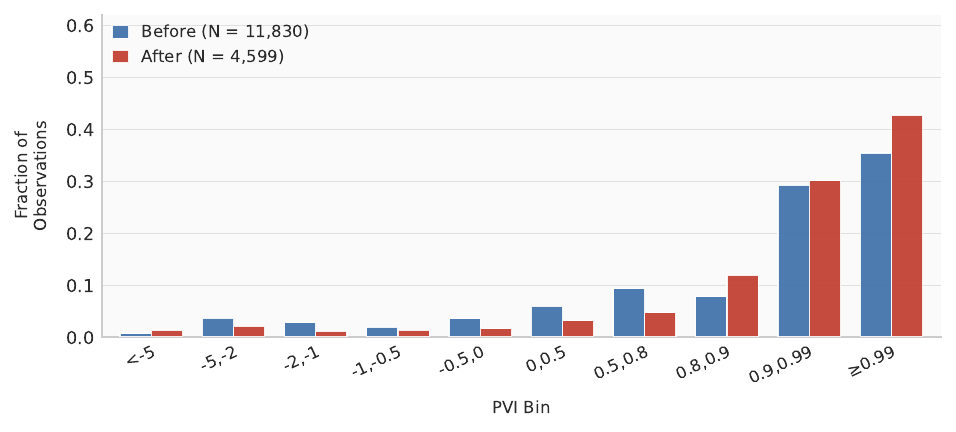}
	\end{subfigure}
	\caption{Distribution of \pvi scores across listings uploaded before and after the ChatGPT release (November 30, 2022), using SELECT/PASS tokens with balanced sampling. Each panel corresponds to a different labeling model. The $y$-axis reports the fraction of observations in each bin.}
	\label{fig:pvi_histogram}
\end{figure}

\section{Economic Impact}
\label{appendix:economic_impact}

Are mecha-nudges having an economic impact? 
In the Etsy context, economic impact can be quantified in many different ways, including more sales, a higher willingness-to-pay for the same product, a lower return rate, etc.
Unfortunately, our dataset does not contain such metrics.
The one noisy cumulative proxy that \textit{is} available in our dataset is the number of product reviews.
Anecdotally, users on a forum for Etsy sellers report that 10--40\% of purchases yield reviews, which are publicly visible on the listing \citep{reddit_etsy_review_rate_2022,reddit_etsycommunity_review_rate_2025}. 
If mecha-nudges help drive product sales, then all else held constant, product listings with more machine-usable information should have more reviews in the post-ChatGPT period (but not in the pre-ChatGPT period).
This is indeed the case.

In Table \ref{tab:review_pvi} (Panel A), we run a PPML regression of per-listing review counts against \pvi.
In the pre-ChatGPT period, an additional bit of machine-usable information \textit{decreases} the number of product reviews by $\approx$22.9\% even after controlling for seller fixed-effects and the age of the product listing\footnote{We include $\log(\text{age})$ as a covariate because reviews do not accumulate linearly; successful older listings can have exposure that compounds over time.}.
This is surprisingly large, especially considering the non-significant decline in human-usable information after the release of ChatGPT.
However, given that LLMs were not widely used by consumers in the pre-ChatGPT period, it could simply be the case that anything that increases machine-usability \textit{within the pre-period} did so in a way that was unappealing to human buyers, with the issue being rectified in the post-period.
In the post-ChatGPT period, an additional bit of machine-usable information \textit{increases} product reviews by $\approx$2.4\% after controls such as listing age.
However, the latter may be an under-estimate given that the prevalence of agentic AI also grew over this time, strengthening the downstream impact of any mecha-nudge. 

Based on the premise that mecha-nudging is probably concentrated in listings for which the agent curation decision $Z_M = 1$ (i.e., the decision=SELECT listings), we then estimate the effect separately for them (Table \ref{tab:review_pvi}, Panel B).
For selected listings in the pre-ChatGPT period, an additional bit of machine-usable information decreases the number of reviews by 25.2\%; in the post-ChatGPT period, it increases the number of
reviews by 81.9\%, though the effect diminishes to 15.5\% after controlling for seller fixed-effects and listing age.
In contrast, there is no significant change for listings with a pass decision.
For a product in the Pet Supplies category, where the average increase in machine-usable information was 0.3775 bits from the pre- to post-period, the corresponding rise in review count---and in sales, assuming the purchase-to-review yield remained constant---would be in the range of 5.8--30.9\%.
For a product in the Art \& Collectibles category, which saw almost no increase in machine-usable information, the average rise in review count would be no greater than 1.1\%.

\begin{table*}[!htbp] \centering
\begin{threeparttable}
\caption{PVI and Number of Product Reviews}
\label{tab:review_pvi}
\begin{tabular}{@{\extracolsep{34pt}}lccc}
\\[-1.8ex]\hline
\hline \\[-1.8ex]
 & \multicolumn{3}{c}{\textit{Dependent Variable: Number of Product Reviews}} \\
\cline{2-4}
\\[-1.8ex] & \multicolumn{1}{c}{(1)} & \multicolumn{1}{c}{(2)} & \multicolumn{1}{c}{(3)} \\
 & Pooled & Seller FE & Seller FE \\
 &  &  & $+\log(\text{age})$ \\
\hline \\[-1.8ex]
\multicolumn{4}{l}{\textit{Panel A}} \\
\\[-1.4ex]
 PVI & $-0.1475$ & $-0.2700$\rlap{$^{***}$} & $-0.2599$\rlap{$^{**}$} \\
 & $(0.0906)$ & $(0.0861)$ & $(0.1029)$ \\
 PVI $\times$ after & $+0.7759$ & $+0.3516$\rlap{$^{***}$} & $+0.2835$\rlap{$^{***}$} \\
 & $(0.6607)$ & $(0.0895)$ & $(0.1088)$ \\
 Observations & 51,033 & 15,030 & 15,030 \\
\\[1.5ex]
\multicolumn{4}{l}{\textit{Panel B (stratified by SELECT/PASS decision)}} \\
\\[-1.4ex]
 PVI $\times$ decision=SELECT & $-0.2900$\rlap{$^{***}$} & $-0.2002$\rlap{$^{**}$} & $-0.0828$ \\
 & $(0.0802)$ & $(0.1014)$ & $(0.0991)$ \\
 PVI $\times$ decision=PASS & $+0.2445$ & $-0.1331$ & $-0.1571$ \\
 & $(0.1630)$ & $(0.1747)$ & $(0.1542)$ \\
 PVI $\times$ after $\times$ decision=SELECT & $+0.8881$\rlap{$^{***}$} & $+0.5615$\rlap{$^{***}$} & $+0.2273$\rlap{$^{*}$} \\
 & $(0.2230)$ & $(0.2055)$ & $(0.1336)$ \\
 PVI $\times$ after $\times$ decision=PASS & $+0.1909$ & $+0.0770$ & $+0.0766$ \\
 & $(0.5268)$ & $(0.2711)$ & $(0.2224)$ \\
\\[-1.4ex]
 Observations & 51,033 & 12,867 & 12,867 \\
\\[-1.0ex]
\hline
\hline \\[-1.8ex]
\textit{Note:} & \multicolumn{3}{r}{$^{*}$p$<$0.1; $^{**}$p$<$0.05; $^{***}$p$<$0.01} \\
\end{tabular}
\begin{tablenotes}[flushleft]
\small
\item[a] PPML (Poisson pseudo-MLE) regression of per-listing review counts against PVI. Data is Gemma-3-27B-labeled, class-balanced listings ($\text{after}=1$ denotes post-ChatGPT). To translate coefficients into a percentage change in the number of reviews, take $(\exp(\text{coefficient}) - 1) \times 100$. Fine-tuning model: Llama-3.1-8B-Instruct. Labeling model: Gemma-3-27B-IT.
\item[b] The Seller FE column drops sellers with one product or whose products all have zero reviews, so $N$ is the effective post-drop sample; Panel B's fixed effects are seller$\times$decision, so more singletons drop and its $N$ is below Panel A's. The $\log(\text{age})$ column adds log days from listing creation to time of scrape.
\item[c] Panel A: In the pre-ChatGPT period, an additional bit of machine-usable information \textit{decreases} the number of product reviews by $(\exp(-0.2599) - 1) \times 100 \approx 22.9\%$ after controlling for seller fixed-effects and listing age. In the post-ChatGPT period, the effect flips but is modest in size: an additional bit of machine-usable information \textit{increases} the number of reviews by $(\exp(-0.2599 + 0.2835) - 1) \times 100 \approx 2.4\%$.
\item[d] Panel B re-estimates the same specification separately for listings the labeling model would SELECT versus PASS, based on the premise that the effects of PVI change are ostensibly concentrated among the former. Indeed, in the pre-ChatGPT period, an additional bit of machine-usable information \textit{decreases} the number of reviews for selected listings by 25.2\% (Column 1); in the post-ChatGPT period, it \textit{increases} the number of reviews by 81.9\%. The effect diminishes to 15.5\% (Column 3) after controlling for seller fixed-effects and listing age.
\end{tablenotes}
\end{threeparttable}
\end{table*}

\section{Mechanisms}
\label{appendix:mechanisms}

\begin{table*}[htbp]
\centering
\caption{Token Importance via Counterfactual Ablation ($N\geq25$)}
\label{tab:token_ablation_n25}
\begin{threeparttable}
\resizebox{\textwidth}{!}{%
\begin{tabular}{p{3cm}rrrrrrrcp{1.5cm}}
\toprule
Token & $N$ & Mean PVI (with) & Mean PVI (without) & $\Delta$ PVI & Pos.\ ($Z_M$) & Pos.\ (Pred.) & Pos.\ (w/o) & Pol. & Source \\
\midrule
prolific & 27 & 0.7866 & 0.0276 & 0.7590 & 96\% & 100\% & 74\% & + & O \\
splitting & 29 & 0.7887 & 0.0481 & 0.7406 & 76\% & 79\% & 62\% & $-$ & O \\
qt & 35 & 0.8209 & 0.1129 & 0.7080 & 26\% & 29\% & 29\% & + & V \\
happier & 35 & 0.7906 & 0.1088 & 0.6818 & 91\% & 86\% & 77\% & + & O,V \\
junk & 98 & 0.8007 & 0.1649 & 0.6358 & 27\% & 24\% & 33\% & $-$ & O \\
snazzy & 30 & 0.9528 & 0.3263 & 0.6265 & 10\% & 13\% & 30\% & + & O \\
dripping & 26 & 0.9498 & 0.3329 & 0.6170 & 42\% & 42\% & 35\% & $-$ & O \\
accident & 27 & 0.8245 & 0.2104 & 0.6141 & 52\% & 48\% & 48\% & $-$ & V \\
relieve & 45 & 0.9625 & 0.3900 & 0.5725 & 51\% & 51\% & 53\% & + & V \\
simplistic & 36 & 0.7163 & 0.1780 & 0.5383 & 64\% & 61\% & 67\% & $-$ & O \\
incomplete & 27 & 0.8177 & 0.2842 & 0.5336 & 26\% & 22\% & 41\% & $-$ & O \\
oddities & 72 & 0.8413 & 0.3122 & 0.5291 & 85\% & 81\% & 75\% & $-$ & O \\
scarce & 75 & 0.7284 & 0.2484 & 0.4800 & 83\% & 81\% & 65\% & $-$ & O \\
forged & 39 & 0.6663 & 0.1899 & 0.4764 & 74\% & 72\% & 62\% & $-$ & O \\
intimacy & 67 & 0.9863 & 0.5139 & 0.4724 & 97\% & 97\% & 79\% & + & O \\
unwanted & 56 & 0.8972 & 0.4284 & 0.4688 & 68\% & 66\% & 57\% & $-$ & O,V \\
jeweler & 93 & 0.8342 & 0.3675 & 0.4667 & 82\% & 86\% & 69\% & 0 & E \\
dying & 62 & 0.9652 & 0.4985 & 0.4667 & 42\% & 42\% & 48\% & $-$ & O \\
employee & 29 & 0.8865 & 0.4434 & 0.4431 & 31\% & 28\% & 41\% & 0 & E \\
snags & 47 & 0.7157 & 0.2741 & 0.4416 & 30\% & 26\% & 40\% & $-$ & O \\
\addlinespace[2pt]
\multicolumn{10}{c}{$\vdots$} \\
\addlinespace[2pt]
attracts & 25 & -0.1578 & 0.8384 & -0.9962 & 80\% & 76\% & 84\% & + & V \\
fissures & 31 & -0.1604 & 0.4953 & -0.6557 & 61\% & 94\% & 71\% & $-$ & O \\
sincere & 27 & -0.1428 & 0.4956 & -0.6384 & 37\% & 22\% & 41\% & + & O,V \\
radiance & 45 & 0.3887 & 0.9584 & -0.5698 & 71\% & 67\% & 71\% & + & O,V \\
cheery & 42 & 0.1782 & 0.7451 & -0.5669 & 43\% & 40\% & 40\% & + & O,V \\
barrier & 49 & 0.3701 & 0.9291 & -0.5590 & 39\% & 29\% & 39\% & $-$ & V \\
unfortunate & 59 & 0.1067 & 0.6621 & -0.5555 & 59\% & 53\% & 53\% & $-$ & O,V \\
inflammation & 25 & 0.1626 & 0.7129 & -0.5503 & 40\% & 28\% & 32\% & $-$ & O \\
brightest & 31 & 0.0977 & 0.6325 & -0.5348 & 65\% & 65\% & 61\% & + & O,V \\
majesty & 40 & 0.1607 & 0.6904 & -0.5297 & 80\% & 68\% & 68\% & + & O \\
administration & 34 & 0.4178 & 0.9473 & -0.5295 & 24\% & 21\% & 24\% & 0 & E \\
lobby & 48 & 0.2878 & 0.7656 & -0.4778 & 62\% & 69\% & 58\% & + & V \\
favored & 31 & 0.2767 & 0.7415 & -0.4648 & 81\% & 74\% & 71\% & + & O,V \\
spacious & 31 & 0.4596 & 0.9232 & -0.4637 & 39\% & 45\% & 35\% & + & O \\
wound & 57 & 0.4115 & 0.8686 & -0.4571 & 75\% & 74\% & 75\% & $-$ & O \\
adorns & 29 & 0.3510 & 0.8066 & -0.4556 & 66\% & 66\% & 69\% & + & V \\
mature & 35 & 0.1879 & 0.6261 & -0.4383 & 49\% & 40\% & 46\% & + & O,V \\
sham & 25 & 0.4116 & 0.8472 & -0.4355 & 32\% & 24\% & 28\% & $-$ & O \\
limits & 56 & 0.4557 & 0.8893 & -0.4336 & 73\% & 68\% & 71\% & $-$ & O \\
economy & 59 & 0.4029 & 0.8271 & -0.4242 & 58\% & 59\% & 58\% & 0 & E \\
\bottomrule
\end{tabular}
}
\begin{minipage}{\textwidth}
\footnotesize
Counterfactual ablation: for each word $t$, listings containing $t$ are identified, $t$ is removed, and the fine-tuned model is re-run on the modified text.
Positive $\Delta \pvi$ indicate that including the word makes the machine behave more predictably; negative $\Delta \pvi$ means it behaves less predictably.
$N$ is the number of listings containing the token; only tokens with $N\geq25$ are included. 
Pos.\ ($Z_M$) is the fraction of listings containing the token where the constructed label is the positive class (SELECT).
Pos.\ (Pred.) is the fraction where the classifier predicts the positive class. 
Pos.\ (w/o) is the predicted positive rate after ablating the token. 
Pol.\ reports the sentiment polarity of the token according to the source lexicon (+~positive, $-$~negative). Source abbreviations: O~=~Opinion Lexicon, V~=~VADER, E~=~Empath.
\end{minipage}
\end{threeparttable}
\end{table*}

\paragraph{Token Ablation} 
Although the specific mechanisms driving mecha-nudging are complex (\S\ref{sec:experiments}), we use a counterfactual ablation approach to identify which individual words might drive changes in PVI. 
Because we are interested in the commerce-related vocabulary that sellers use to describe their products, we restrict attention to tokens drawn from three established NLP lexicons: the Opinion Lexicon \citep{hu2004mining}, which covers positive and negative sentiment words; VADER \citep{hutto2014vader}, a sentiment-intensity lexicon; and the business, money, and shopping categories of Empath \citep{fast2016empath}.

For each word $t$ in the combined lexicon that appears in the data, we identify the listings containing $t$, remove it using a word-boundary regular expression, and re-run the fine-tuned model on the modified text. 
Then we calculate:
\begin{equation}
\Delta\text{PVI}(t) = \frac{1}{|\mathcal{L}_t|}\sum_{i \in \mathcal{L}_t} \pvi (x_i \to b_i) - \frac{1}{|\mathcal{L}_t|}\sum_{i \in \mathcal{L}_t} \pvi (x_{i, \neg t} \to b_i)
\end{equation}
where $\mathcal{L}_t$ denotes the set of listings in the sample that contain token $t$. 

Note that $\pvi (x_{i, \neg t} \to b_i)$ is not technically the \pvi of the modified text, since our intervention changes the distribution and would in theory require retraining another pair of content and null models.
Since this is impractical to do for every token, we instead decide to estimate the change in \pvi by just using the models finetuned on the original text.
A positive $\Delta\text{PVI}$ indicates that the token makes the machine more predictable with respect to $Z_M$; a negative value indicates that it makes the machine's decision less predictable. 
Table~\ref{tab:token_ablation_n25} reports the tokens with the largest positive and negative $\Delta\text{PVI}$.
These single-token ablations should be treated as descriptive evidence rather than a causal decomposition of mecha-nudging.

\paragraph{Listing Length} 
The literature has reported LLMs often prefer longer texts \citep{dubois2024lengthcontrolledalpacaeval}.   
In the Etsy data, however, length does not explain the observed post-ChatGPT increase in machine-usable information.
As seen in Figure \ref{fig:lengths}, although there is an increase in median listing length post-ChatGPT, not only is it modest in magnitude, it shows little-to-no correlation with machine-usable information \textit{within} the post-period itself.
Even though machine-usable information peaks in the most recent half-year, the listing length is far below its peak.
This, in addition to the previous regression wherein word count did not have a significant effect on \pvi (Table \ref{tab:controls}), suggests that simply lengthening the listing is not sufficient to explain the observed mecha-nudging.

\begin{figure}[t]
    \centering
\includegraphics[width=\linewidth]{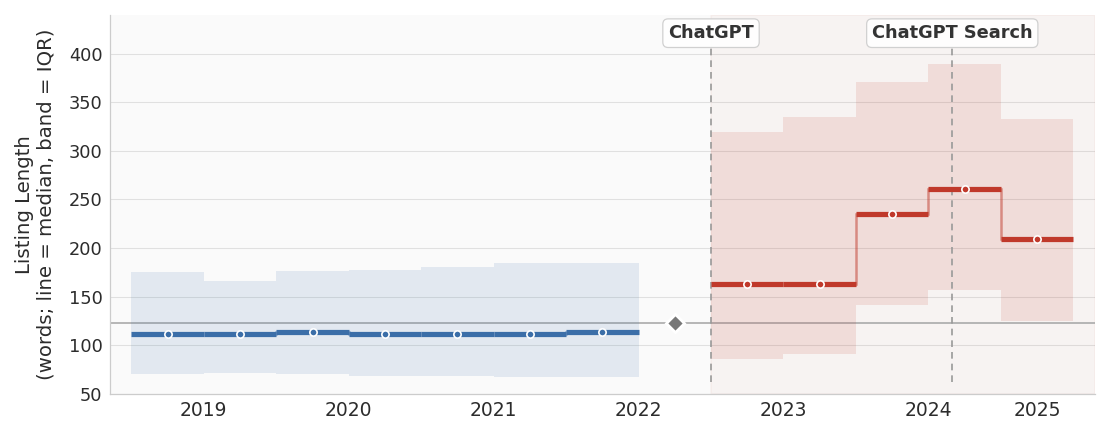}
    \caption{The median length of product listings in each half-year (inter-quartile range captured by shading). Blue denotes the pre-ChatGPT period; red the post-ChatGPT period; and the grey diamond the period just before ChatGPT (Jul-Sep 2022). Note that although the release of ChatGPT coincides with an increase in listing length, it is still modest in magnitude. Moreover, the length rises in 2024 even as the amount of machine-usable information falls (Figure \ref{fig:time_variation}), then starts declining after the release of ChatGPT Search even though the amount of machine-usable information starts climbing again. Mecha-nudging in Etsy listings is much more nuanced than simply adding more content.}
    \label{fig:lengths}
\end{figure}

\section{Human-Usable Information}
\label{appendix:human_eval}
The main experiment shows that machine-usable information
$I_{\mathcal{M}}(X \to Z_M)$ rises significantly after the release of
ChatGPT. 
For this to qualify as a \emph{realized mecha-nudge}
(Definition~\ref{def:realized_mecha_nudge}), the human-side constraint $$I_{\mathcal{H}}(\tau^{*}(X) \to A_H) \ge I_{\mathcal{H}}(X \to A_H) - \epsilon$$ must also hold for a reasonably small $\epsilon$. 
Earlier, we provided indirect real-world evidence that is consistent with $\epsilon$ being small (Section \ref{sec:experiments}); here we report the details of the human survey we conducted using historical listings to directly estimate $I_{\mathcal{H}}(X \to A_H)$.
A transformation can be a mecha-nudge even if it increases---rather than decreases---the amount of human-usable information, so long as it systematically influences machine behavior.
In practice, however, we are especially interested in machine-skewed changes, where machine-usable information rises more than human-usable information.

\subsection{Survey Design}
\label{ssec:survey_design}

\begin{figure}[!htbp]
\centering
\includegraphics[width=0.48\linewidth, keepaspectratio]{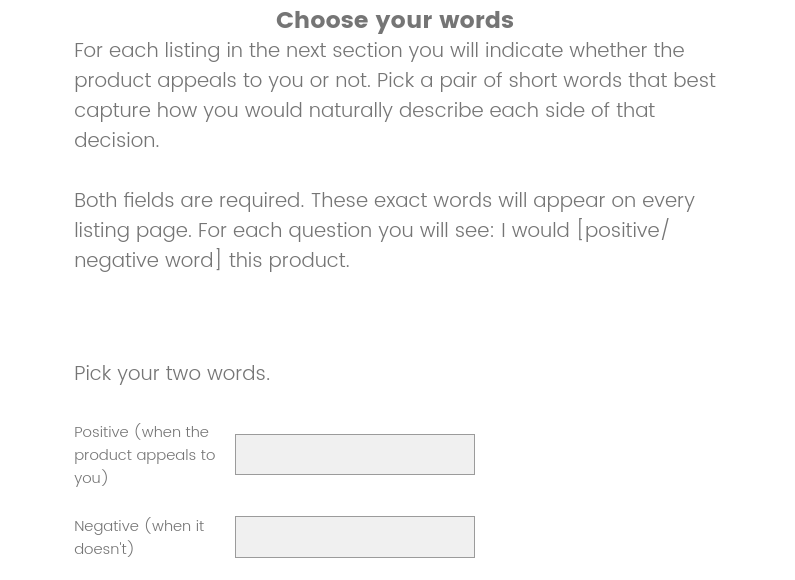}
\includegraphics[width=0.48\linewidth, keepaspectratio]{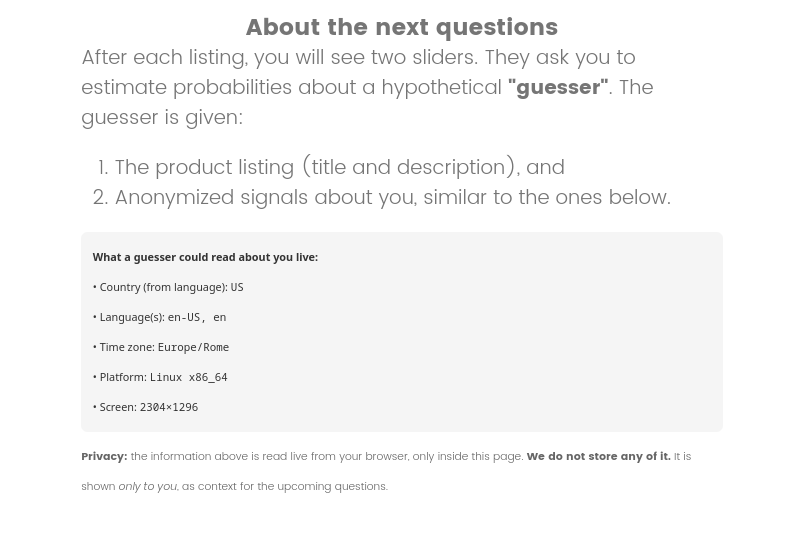}
\includegraphics[width=0.6\linewidth, keepaspectratio]{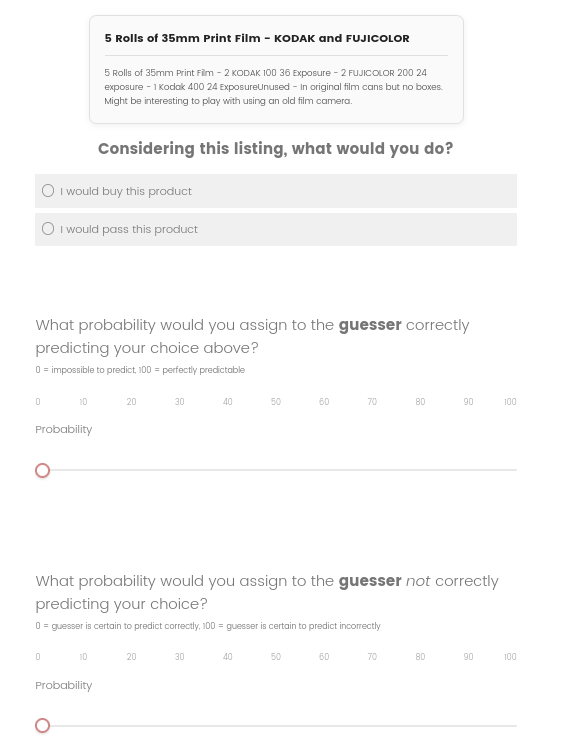}
\caption{Survey screens, clockwise in order of appearance: respondents pick their own
positive/negative action tokens, see a primer on what a human predictor would have access to, and then evaluate twenty listings (one per page). The two sliders are used to filter out respondents without an understanding of probability.}
\label{fig:survey_screens}
\end{figure}

As shown in Figure \ref{fig:survey_screens}, we use a self-elicitation design:
a respondent sees a listing with its title and description untruncated, and chooses an action
$a_h \in \{Y_{\text{pos}}, Y_{\text{neg}}\}$, where the two action labels are
short tokens the respondent supplies themselves at the start of the survey
(e.g.\ ``buy''/``ignore''). Two sliders follow, asking the respondent to estimate:
\begin{itemize}\itemsep0.1em
\item $p_Y \in [0,100]$: the probability that a hypothetical guesser, given the listing and anonymized signals about them, would correctly predict their action
\item $p_{\neg Y} \in [0,100]$: the probability that a hypothetical guesser, given the listing and anonymized signals about them, would \textit{not} correctly predict their action
\end{itemize}
Because any third-party annotator could have no better theory-of-mind for the respondent's actions than the respondent themselves, and because $\mathcal{V}$-entropy seeks to find the best possible predictor, self-elicitation is likely a better means of estimating the human-usable information.
Moreover, giving a third-party annotator sufficient context about the respondent and their preferences would raise privacy issues.
For these reasons, we use self-elicitation instead.

The first slider yields the $\mathcal{H}$-optimal conditional, $g'[x](a_h) = p_Y/100$, evaluated at the realized action. 
The pointwise contribution to the estimate of conditional $\mathcal{V}$-entropy is therefore $-\log_2 g'[x](a_h) = -\log_2(p_Y/100)$.
The second slider is a self-consistency check: since $p_Y, p_{\neg Y}$ are the probabilities of the realized and unrealized actions respectively, a respondent with consistent preferences should satisfy $p_Y + p_{\neg Y} \le 110$ (allowing for some noise).

Each respondent sees 20 unique listings one at a time (ten pre-ChatGPT and ten post-ChatGPT, drawn from the same population used for the $I_{\mathcal{M}}(X \to Z_M)$ estimate; seed $=42$). 
The two action tokens chosen earlier are used to generate the per-listing label (``I would $\langle\text{token}\rangle$
this product'').
We show the respondent their
own browser locale, time zone, and device---these fields are not stored, only shown to make concrete the signals available to a hypothetical guesser.
Figure~\ref{fig:survey_screens} illustrates the relevant pages.

For respondents to estimate the unconditional term $H_\mathcal{H}(A_H)$, we would have to ask them to estimate the propensity of humans overall to act as they would across all listings in the two different time periods.
As forcing respondents to marginalize across humans is not realistic, we assume---based on the stability of Etsy buyer statistics across both periods (\S\ref{sec:experiments})---that the $\mathcal{V}$-entropy does not change.
This means that we can estimate $I_\mathcal{H}(\tau^*(X) \to A_H) - I_\mathcal{H}(X \to A_H)$ by taking the average over $\pvi_\mathcal{H}(\tau^*(x_i) \to a_h) - \pvi_\mathcal{H}(x_i \to a_h)$, which we get from the slider.

After obtaining an `exempt' designation from our institution's IRB board, we recruited 689 respondents on Prolific and applied the per-row attention check described above.
Each respondent was paid \$4 for taking the survey and going through the 20 listings, which they were informed would take 10-15 minutes of their time.
The risks and benefits involved and the terms of confidentiality were described to the respondents.

\subsection{Results}

We applied a per-respondent filter that drops anyone failing more than $20\%$ of their per-row checks. 
This is intended to remove respondents who have a systematically poor understanding of probability and who are setting the sliders randomly to complete the task in as little time as possible.
This leaves $525$ respondents and $10{,}106$ observations. 
The realized action mix is approximately $36/64$ positive/negative both before and after ChatGPT.

We compute the mean for each respondent in each period, then take the average across all respondents.
Because every respondent sees 10 listings from each period, we use the within-subject standard-error estimator: the standard error of the mean is thus the cross-respondent SD of the per-respondent differences divided by $\sqrt{n_{\text{resp}}}$, with $n_{\text{resp}} - 1$ degrees of freedom.

\begin{table}[!htbp] \centering
\begin{threeparttable}
\caption{Human-Side Conditional Log-Score}
\label{tab:human_eval_main}
\begin{tabular}{@{\extracolsep{74pt}}lcc}
\\[-1.8ex]\hline
\hline \\[-1.8ex]
& \multicolumn{2}{c}{\textit{Dependent variable: $\log_2 \hat{g}'[X](A_H)$}} \cr \cline{2-3}
\\[-1.8ex] & \multicolumn{1}{c}{Valid respondents} & \multicolumn{1}{c}{All respondents} \\
\hline \\[-1.8ex]
Pre-ChatGPT Mean & $-$0.891 & $-$0.912 \\
Post-ChatGPT Mean & $-$0.906 & $-$0.923 \\
\hline \\[-1.8ex]
after & $-$0.015 & $-$0.013 \\
& (0.015) & (0.013) \\
\hline \\[-1.8ex]
Respondents & 525 & 680 \\
Observations & 10106 & 11811 \\
\hline
\hline \\[-1.8ex]
\textit{Note:} & \multicolumn{2}{r}{$^{*}$p$<$0.1; $^{**}$p$<$0.05; $^{***}$p$<$0.01} \\
\end{tabular}
\begin{tablenotes}[flushleft]
\small
\item[a] Each cell is the within-respondent contribution to $\log_2 \hat{g}'[X](A_H)$, the negative of the conditional $\mathcal{V}$-entropy for humans. 
Higher values (i.e., less negative) indicate listings that are more predictable for humans. 
The coefficient on \textit{after} is the post minus pre difference in means; positive values would indicate that the post-ChatGPT increase seen for machine-usable information is \emph{also} present for humans. Here the differences are negative but not significantly so.
\item[b] Standard errors in parentheses are paired at the respondent level: each respondent contributes one difference $\bar{x}_{\text{after}} - \bar{x}_{\text{before}}$, and the SE is the cross-respondent SD of those differences divided by $\sqrt{n_{\text{resp}}}$. This is a natural choice of estimator since every respondent sees 10 listings from each period.
\item[c] Sample: 689 Prolific respondents, pooling the April and August waves. \textit{All respondents} includes every respondent who finished the survey; \textit{Valid respondents} drops those where more than $20\%$ of answers failed the per-row attention check $p_Y + p_{\neg Y} \leq 110$.
\end{tablenotes}
\end{threeparttable}
\end{table}

\begin{table}[!htbp] \centering
\begin{threeparttable}
\caption{Robustness to the Respondent-Level Attention Filter}
\label{tab:human_eval_robustness}
\begin{tabular}{@{\extracolsep{34pt}}lcccc}
\\[-1.8ex]\hline
\hline \\[-1.8ex]
& \multicolumn{4}{c}{\textit{Dependent variable: $\log_2 \hat{g}'[X](A_H)$}} \cr \cline{2-5}
\\[-1.8ex] & \multicolumn{4}{c}{\textit{Respondent dropped if bad-row share exceeds}} \\
\\[-1.8ex] & \multicolumn{1}{c}{$>20\%$} & \multicolumn{1}{c}{$>50\%$} & \multicolumn{1}{c}{$>80\%$} & \multicolumn{1}{c}{none} \\
\hline \\[-1.8ex]
Pre-ChatGPT Mean & $-$0.891 & $-$0.912 & $-$0.912 & $-$0.912 \\
Post-ChatGPT Mean & $-$0.906 & $-$0.923 & $-$0.923 & $-$0.923 \\
\hline \\[-1.8ex]
after & $-$0.015 & $-$0.011 & $-$0.013 & $-$0.013 \\
& (0.015) & (0.014) & (0.013) & (0.013) \\
\hline \\[-1.8ex]
Respondents & 525 & 633 & 677 & 680 \\
Observations & 10106 & 11490 & 11802 & 11811 \\
\hline
\hline \\[-1.8ex]
\textit{Note:} & \multicolumn{4}{r}{$^{*}$p$<$0.1; $^{**}$p$<$0.05; $^{***}$p$<$0.01} \\
\end{tabular}
\begin{tablenotes}[flushleft]
\small
\item[a] The attention check $p_Y + p_{\neg Y} \leq 110$ is used for each response; the column titled $>20\%$ drops any respondent failing the attention check on more than $20\%$ of their responses, and the final column removes no individuals  at all (i.e., an effective threshold of $>100\%$). 
\item[b] The coefficient on \textit{after} is the difference in respondent-level means across the pre- and post- ChatGPT periods, with paired standard errors in parentheses.
\item[c] All columns pool the April and August waves of respondents, who are given an identical survey. The $>20\%$ column is the \textit{Valid Respondents} slice of Table~\ref{tab:human_eval_main} and the \textit{none} column its \textit{All Respondents} slice.
\end{tablenotes}
\end{threeparttable}
\end{table}

\begin{table}[!htbp] \centering
\begin{threeparttable}
\caption{Human-Chosen Action Tokens}
\label{tab:human_eval_tokens}
\begin{tabular}{@{\extracolsep{82pt}}lrlr}
\\[-1.8ex]\hline
\hline \\[-1.8ex]
\multicolumn{2}{c}{\textit{Positive}} & \multicolumn{2}{c}{\textit{Negative}} \\
\cline{1-2} \cline{3-4}
\\[-1.8ex] Token & $N$ & Token & $N$ \\
\hline \\[-1.8ex]
like & 97 & dislike & 94 \\
good & 62 & bad & 77 \\
buy & 58 & skip & 47 \\
love & 57 & hate & 41 \\
purchase & 18 & pass & 26 \\
nice & 18 & no & 22 \\
appealing & 18 & ignore & 22 \\
yes & 17 & not buy & 18 \\
positive & 14 & unappealing & 14 \\
great & 12 & negative & 14 \\
enjoy & 9 & avoid & 10 \\
cool & 8 & ugly & 9 \\
want & 7 & reject & 9 \\
amazing & 6 & boring & 7 \\
awesome & 6 & not like & 6 \\
\hline
\hline \\[-1.8ex]
\end{tabular}
\begin{tablenotes}[flushleft]
\small
\item Each respondent supplies one positive and one negative action token at the start of the survey; these are piped into the per-listing multiple-choice options as ``I would \emph{token} this product''. $N$ is the number of respondents whose self-chosen token matched (case-folded). Respondents: 689; unique positive tokens: 218, unique negative tokens: 230, unique pairs: 364.
\end{tablenotes}
\end{threeparttable}
\end{table}

As seen in Table~\ref{tab:human_eval_main}, among valid respondents, post-ChatGPT listings carry on average only about $0.015$ fewer bits of human-usable information about the realized action than pre-ChatGPT ones, compared to a machine-side gain of $+0.143$ bits. 
More importantly, the difference is non-significant at any confidence level.
The magnitude is small enough to fit comfortably under any reasonable choice of $\epsilon$ in Definition~\ref{def:realized_mecha_nudge}: i.e., the confidence interval is consistent with invariance on the human side.
The pre- and post-ChatGPT distributions of the pointwise conditional human-entropy are nearly identical, with a small shift of mass out of the upper-middle bins (roughly $0.5$ to $0.9$) into the near-zero bin ($-0.5$ to $0$) and negative tail.
This is consistent with, although not dispositive of, little human-side degradation. 

Table~\ref{tab:human_eval_robustness} re-estimates the difference in human-usable information under alternative settings of the attention filter, dropping only respondents whose share of failed checks exceeds $20\%$, $50\%$, $80\%$, and $100\%$ (i.e., dropping none at all).
Results do not vary: relaxing the threshold increases the sample from $525$ to $680$ respondents, but the estimate stays between $-0.015$ and $-0.011$ bits and the paired standard error moves only from $0.015$ to $0.013$. 
No column reaches a conventional significance level, with $p$ ranging from $0.309$ to $0.423$.

Finally, Table~\ref{tab:human_eval_tokens} lists the action tokens most chosen by respondents for the positive and negative options. 
Naturally, these have a Zipfian distribution ($218$ unique positive tokens, $230$ unique negative tokens), with ``like'', ``good'', ``buy'' and ``love'' dominating the positive column and ``dislike'', ``bad'', ``skip'' and ``hate'' dominating the negative column.
Allowing respondents to supply their own tokens reduces the framing that a survey environment imposes, and better approximates what respondents' decisions would be in the real world.

\begin{figure}[t!]
	\centering
    \includegraphics[width=0.85\textwidth]{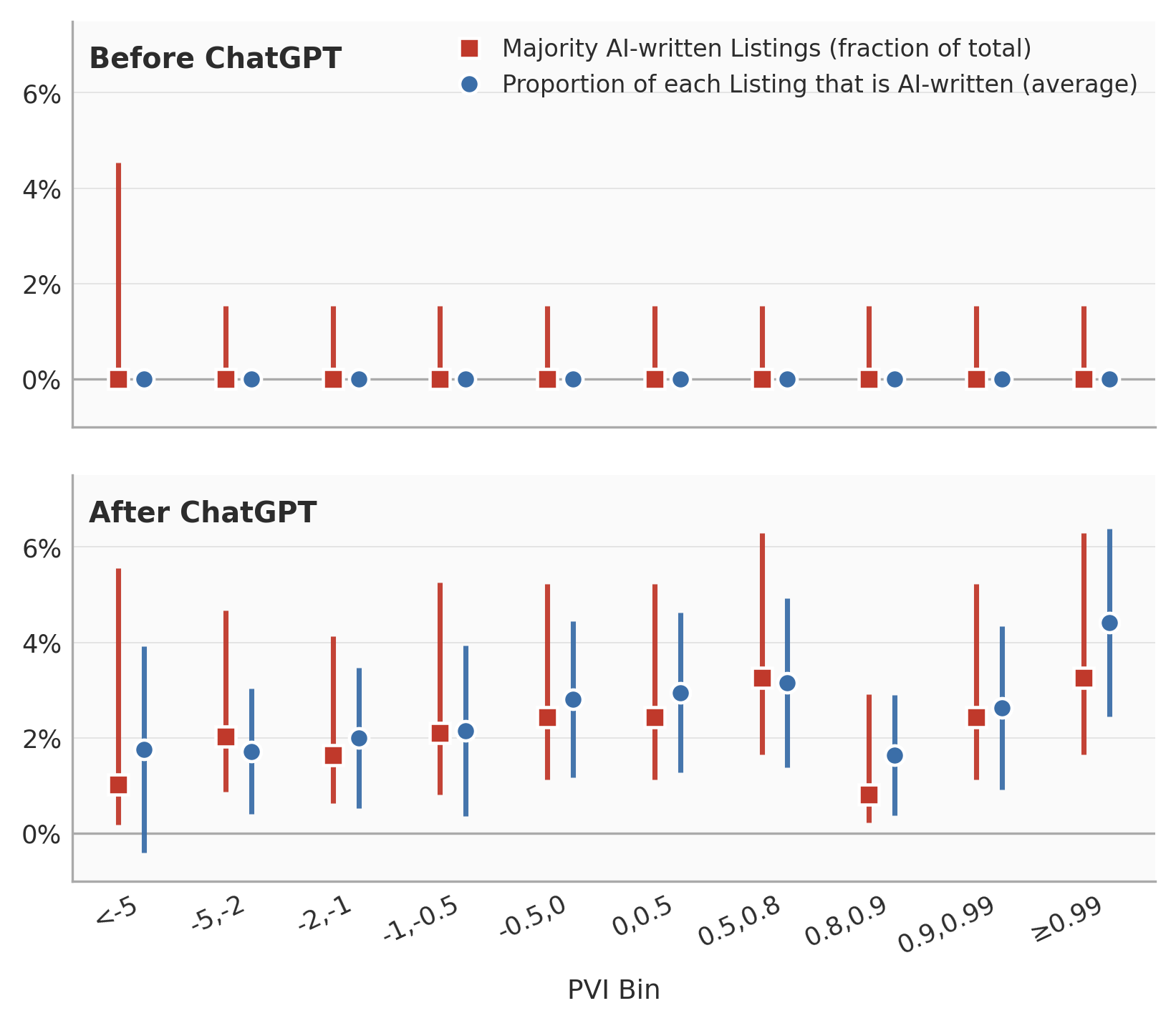}
	\caption{AI-text detection for each \pvi bin, scored with \texttt{pangram-4}. Each panel reports two measures per bin: the share of listings with at least half of the text flagged as AI-written, and the average proportion of each listing that is AI-written. Vertical bars are 95\% confidence intervals (Wilson for the fractions and normal for the means). Bins match those used for the \pvi histograms in Figure \ref{fig:pvi_histogram}. Not only is the prevalence of AI-written content low overall, but it also does not vary much with \pvi; just increasing the amount of machine-written content does not necessarily increase the amount of machine-usable information.}
	\label{fig:pangram_detection}
\end{figure}

\section{AI Text Detection}
\label{appendix:ai_detection}

The rephrasing control in Appendix \ref{appendix:controls} asks what happens when an LLM rewrites a listing.
Here we ask how much of the Etsy text is LLM-written in the first place.
We score listings with \texttt{pangram-4}, the most recent classifier from Pangram, a well-known commercial AI detector that has also been independently validated \citep{jabarian2025artificial}.
Listings are sampled within each (period, \pvi bin) cell, using the same bins as the \pvi histograms, so that the whole \pvi distribution is represented. We analyze $2{,}295$ pre-ChatGPT and $2{,}256$ post-ChatGPT listings.

In the pre-period, 0\% of listings are flagged as being mostly AI-written.
Given that LLMs were not widely available to the public at the time, this is expected and provides some reassurance that the detector is not finding false positives.
In the post-period, less than 4\% of listings are flagged as being mostly AI-written, and the rate is nearly flat across the \pvi distribution (Figure \ref{fig:pangram_detection}).
Whatever separates high-\pvi listings from low-\pvi ones, it is not that a machine wrote them.

\section{Related Work}
\label{appendix:related}

We begin with the empirical and conceptual literature on nudges, then discuss three frameworks (Bayesian persuasion, rational inattention, and salience) that have been used to model how information design shapes decisions, and explain why they fall short in the AI agent setting. 
We then introduce $\mathcal{V}$-information, the observer-relative measure of usable information on which our framework builds.

Nudges are a central theme in behavioral economics. Following \citet{thalerNudgeImprovingDecisions2008}, we use the term \emph{nudge} to mean a feature of choice architecture that predictably shifts behavior without removing options or materially changing economic incentives, and that remains easy to avoid. 
For example, automatic enrollment in retirement plans sharply increases participation and the default options anchor the actual contribution amounts \citep{madrianPowerSuggestionInertia2001}.
Nudges are most often discussed informally, as context-specific interventions. 
However, they can be
mathematically formalized under some frameworks, which we discuss below: Bayesian persuasion (nudges as signal design), rational inattention (nudges as changes in information acquisition costs), and salience (nudges as context-dependent attention and weighting).

\paragraph{Bayesian Persuasion}

A natural benchmark is the framework of Bayesian persuasion, which asks how a sender should design a signaling policy to influence receivers’ beliefs and actions about some underlying quality of an item. An item can be a product, but also, in a famous example of this theory, a defendant’s guilt, with the sender being a prosecutor and the receiver the judge. The signaling structure might be public, as in the seminal paper \citet{kamenicaBayesianPersuasion2011a}, or private, as in \citet{arieliPrivateBayesianPersuasion2019a}.

In the canonical framing, a sender chooses an information structure (signals) to influence a receiver’s action $a$ given state $s$ and utilities; the receiver best-responds to the posteriors induced by the signal. \citet{kamenicaBayesianPersuasion2011a} characterize optimal signal structures under known priors, action sets, and payoffs. If, on the other hand, private messages are available, then the optimal policy is well-defined in the case of no receiver payoff externalities and the sender’s additive utility over receiver responses. With conditionally independent signals, a policy that is separately optimal for each receiver is also globally optimal \citep{arieliPrivateBayesianPersuasion2019a}.

However, mecha-nudges are not solely private or public signals. One way to model this type of interaction is through the literature on Bayesian persuasion with leakages \citet{haghtalabLeakageRobustBayesianPersuasion2024}. More broadly, the framework requires a Bayesian update conditioned on the signal, which is intractable if the receiver is an LLM operating over free-form text, as in our setup.

\paragraph{Rational Inattention} Rational inattention offers another lens for analyzing why humans may not verify the presence of hidden messages. \citet{simsImplicationsRationalInattention2003} models settings in which learning about signals is costly, a framework that can be represented as a generalized multinomial logit \citep{matejkaRationalInattentionDiscrete2015}. Because investigating a nudge is costly, rational inattention predicts that individuals will not always have an incentive to verify whether hidden signals are present. This literature is closely related to the Bayesian persuasion literature discussed above: Bayesian persuasion can also incorporate costs of persuasion \citep{gentzkowCostlyPersuasion2014}, and in both theories, signals are costly (see \S4 in \citet{kamenicaBayesianPersuasion2011a}). Rational inattention provides a useful account of why mecha-nudges can persist---human verification is expensive---but, like Bayesian persuasion, it requires specifying information costs and decision structures that are difficult to specify for LLM-based agents.

\paragraph{Salience}
Another perspective from theory frames nudges as \emph{context-dependent attention}, in which select features capture the attention of decision-makers. 
In the salience framework by \citet{bordaloSalienceConsumerChoice2013}, consumers place excessive emphasis on particular aspects of products (e.g., price or quality). 
Since salience is specified in relative terms, the same item may receive varied evaluations depending on which menu of options it appears in, creating effects similar to reframing or introducing decoy options even when possible actions and rewards remain fixed. 
This approach is distinct from the behavior seen in frameworks where companies deliberately hide data and certain individuals initially fail to notice relevant details, for instance \citet{gabaixShroudedAttributesConsumer2006}.

In the context of model concealment, distortions occur when some attributes are hidden or not readily observable; in the salience framework, distortions result from allocating attention \emph{within} some observed menu, depending on what draws attention. 
This explains how choice architecture influences outcomes: it makes some information more accessible to the decision-maker.
Still, the traditional framework assumes that the decision-maker is a human selecting from a set of well-specified choices, which may not extend to AI agents processing free-form text, the subject of our study.

\end{document}